\documentclass[11pt, letterpaper]{cmu}

\usepackage[all]{hypcap}
\usepackage[comma,numbers,sort,compress]{natbib}
\usepackage{hyperref}[citecolor=magenta]

\hypersetup{
    colorlinks = true,
    citecolor = {magenta},
}

\usepackage{microtype}
\usepackage{graphicx}
\usepackage{subfigure}
\usepackage{booktabs}
\usepackage{float}
\usepackage{amsmath}
\usepackage{amssymb}
\usepackage{mathtools}
\usepackage{amsthm}
\usepackage{mathrsfs}
\usepackage{nicefrac}
\usepackage{dsfont}
\usepackage{enumitem}
\usepackage{float}
\usepackage[utf8]{inputenc}

\usepackage{threeparttable}
\usepackage{booktabs}

\definecolor{myblue}{rgb}{0.1,0.4,0.9}

  {\end{tcolorbox}}
\usepackage[capitalize,noabbrev]{cleveref}
\usepackage{subcaption}
\usepackage{wrapfig}
\usepackage{lipsum}
\usepackage{listings}

\usepackage{amsmath}
\usepackage{amssymb}
\usepackage{mathtools}
\usepackage{amsthm}
\usepackage{bbm}
\usepackage{fleqn, tabularx}
\usepackage{multirow}
\usepackage[export]{adjustbox}
\usepackage{colortbl}

\theoremstyle{plain}
\newtheorem{theorem}{Theorem}[section]

\newtheorem{lemma}[theorem]{Lemma}

\theoremstyle{definition}
\newtheorem{definition}[theorem]{Definition}
\newtheorem{assumption}[theorem]{Assumption}
\theoremstyle{remark}

\usepackage{algpseudocode}
\usepackage{setspace}

\usepackage{color}
\definecolor{deepblue}{rgb}{0,0,0.5}
\definecolor{deepred}{rgb}{0.6,0,0}
\definecolor{deepgreen}{rgb}{0,0.5,0}
\definecolor{boost_correct_to_correct}{HTML}{66C2A5}
\definecolor{default_correct_to_correct}{HTML}{fc8d62}
\definecolor{dup_correct_to_correct}{HTML}{8da0cb}
\definecolor{new_correct_to_correct}{HTML}{e78ac3}

\newcommand\pythonstyle{\lstset{
basicstyle=\ttfamily\footnotesize,
language=Python,
morekeywords={self, clip, exp, mse_loss, uniform_sample, concatenate, logsumexp},              
keywordstyle=\color{deepblue},
emph={MyClass,__init__},          
emphstyle=\color{deepred},   
stringstyle=\color{deepgreen},
frame=single,                       
showstringspaces=false
}}

\lstnewenvironment{python}[1][]
{
\pythonstyle
\lstset{#1}
}
{}

\newcommand\pythoninline[1]{{\pythonstyle\lstinline!#1!}}

\usepackage{amsmath,amsfonts,bm}

\def\eqref#1{Eq.~\ref{#1}}
\def\1{\bm{1}}

\def\rr{{\textnormal{r}}}

\DeclareMathAlphabet{\mathsfit}{\encodingdefault}{\sfdefault}{m}{sl}
\SetMathAlphabet{\mathsfit}{bold}{\encodingdefault}{\sfdefault}{bx}{n}

\newcommand{\E}{\mathbb{E}}

\newcommand{\Var}{\mathrm{Var}}

\newcommand{\methodname}{\texttt{floq}}

\newcommand{\bz}{\mathbf{z}}

\newcommand{\bs}{\mathbf{s}}
\newcommand{\ba}{\mathbf{a}}

\newcommand{\bx}{\mathbf{x}}

\makeatletter
\def\mathcolor#1#{\@mathcolor{#1}}
\def\@mathcolor#1#2#3{%
  \protect\leavevmode
  \begingroup
    \color#1{#2}#3%
  \endgroup
}
\makeatother

\usepackage[textsize=tiny]{todonotes}
\usepackage{algorithm}
\usepackage{amssymb}

\usepackage[skins,theorems]{tcolorbox}

\usepackage{crossreftools}
\usepackage[suppress]{color-edits}

\usepackage{dsfont}
\usepackage{nicefrac}

\newtcolorbox{analysisbox}[1][]{
    enhanced jigsaw,
    colback=white,
    colframe=blue!75!black,
    fonttitle=\bfseries,
    boxsep=5pt,
    left=5pt,
    right=5pt,
    top=5pt,
    bottom=5pt,
    title=#1,
}

\tcbset{
  aibox/.style={
    width=\linewidth,
    top=8pt,
    bottom=4pt,
    colback=rliableolive!8!white,
    colframe=black,
    colbacktitle=black,
    enhanced,
    center,
    attach boxed title to top left={yshift=-0.1in,xshift=0.15in},
    boxed title style={boxrule=0pt,colframe=white,},
  }
}
\newtcolorbox{AIbox}[2][]{aibox,title=#2,#1}
\definecolor{lightblue}{rgb}{0.22,0.45,0.70}% light blue

\usepackage{algorithm}      % float wrapper for algorithms
\usepackage{algpseudocode}  % pseudocode commands (\For, \Function, etc.)
\usepackage[endLComment=,italicComments=false]{algpseudocodex} % nice extras
\usepackage{tcolorbox}      % boxes
\usepackage{xcolor}
\usepackage{pifont}
\usepackage{soul}
\definecolor{highlightmistake}{RGB}{255, 179, 179}
\definecolor{highlightcorrect}{RGB}{179, 255, 179} 

\definecolor{rliableolive}{HTML}{BBCC33}
\definecolor{rliableblue}{HTML}{77AADD}
\definecolor{rliablered}{HTML}{EE8866}

\definecolor{myblue}{rgb}{0.1,0.4,0.9}
\newcommand{\shadebyYR}[3]{%
  \begingroup
  \pgfmathsetmacro{\den}{max(#2-#3, 1e-6)}%
  \pgfmathsetmacro{\t}{max(0, min(1, (#2-#1)/\den))}%
  \pgfmathsetmacro{\rr}{1}%
  \pgfmathsetmacro{\gg}{1 - 0.25*\t}%
  \pgfmathsetmacro{\bb}{1 - \t}%
  \definecolor{tmpShade}{rgb}{\rr,\gg,\bb}%
  \setlength{\fboxsep}{1pt}%
  \colorbox{tmpShade}{$#1$}%
  \endgroup
}

\title{What Does Flow Matching Bring To TD Learning?}

\author[1]{Bhavya Agrawalla}
\author[2]{Michal Nauman}
\author[1]{Aviral Kumar}

\affil[1]{Carnegie Mellon University}
\affil[2]{University of Warsaw}

\correspondingauthor{bbagrawa@andrew.cmu.edu}

\begin{abstract}
\textbf{Abstract:} Recent work shows that flow matching can be effective for scalar Q-value function estimation in reinforcement learning (RL), but it remains unclear why or how this approach differs from standard critics. Contrary to conventional belief, we show that their success is not explained by distributional RL, as explicitly modeling return distributions can reduce performance. Instead, we argue that the use of integration for reading out values and dense velocity supervision at each step of this integration process for training improves TD learning via two mechanisms. First, it enables robust value prediction through \emph{test-time recovery}, whereby iterative computation through integration dampens errors in early value estimates as more integration steps are performed. This recovery mechanism is absent in monolithic critics. Second, supervising the velocity field at multiple interpolant values induces more \emph{plastic} feature learning within the network, allowing critics to represent non-stationary TD targets without discarding previously learned features or overfitting to individual TD targets encountered during training. 
We formalize these effects and validate them empirically, showing that flow-matching critics substantially outperform monolithic critics (2$\times$ in final performance and around 5$\times$ in sample efficiency) in settings where loss of plasticity poses a challenge e.g., in high-UTD online RL problems, while remaining stable during learning.
\\\\
\textbf{Code and runs can be found at}: \href{https://github.com/CMU-AIRe/floq/tree/main/what_does_flow_matching}{https://github.com/CMU-AIRe/floq/tree/main/what-does-flow-matching.}
\end{abstract}

\begin{document}

\maketitle

\vspace{-0.3cm}
\section{Introduction}
\label{sec:introduction}
\vspace{-0.2cm}

\begin{wrapfigure}{r}{0.51\textwidth}
\vspace{-0.8cm}
\centering
\includegraphics[width=0.99\linewidth]{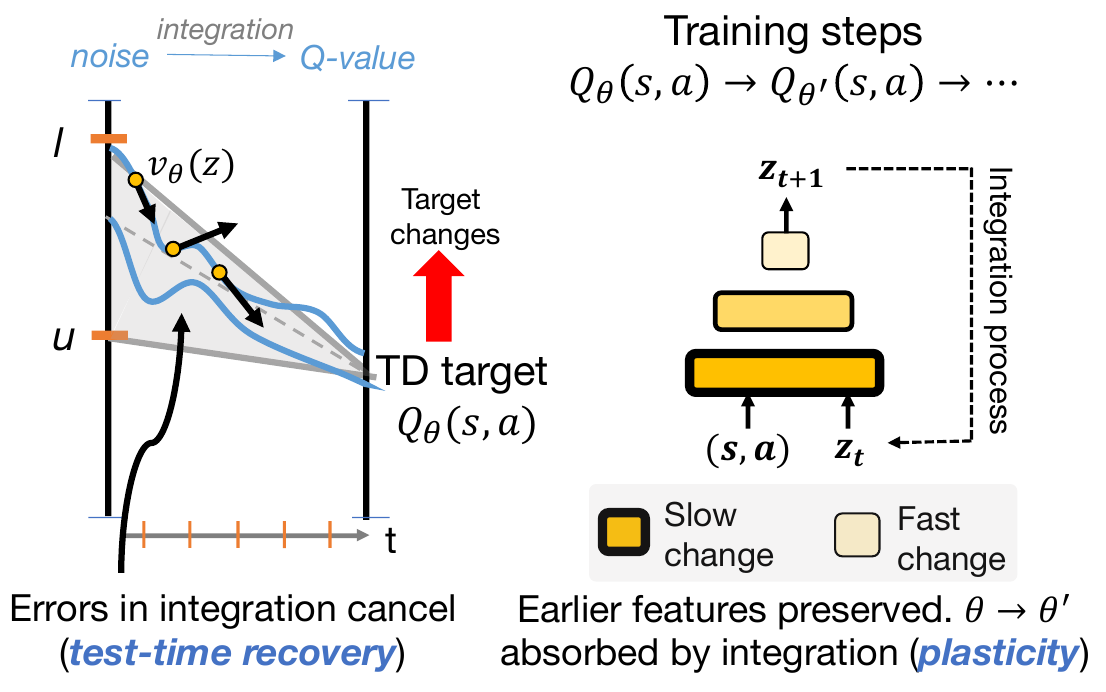}
\vspace{-0.65cm}
\caption{\footnotesize \emph{\textbf{Flow-matching critics with non-stationary TD targets.}} Values are computed by integrating a learned velocity field over multiple steps. This iterative process enables \textcolor{lightblue}{\textbf{test-time recovery}} (left), where errors in early integration steps are dampened by later steps. 
When TD targets change ($\theta \rightarrow \theta'$), the integration dynamics can absorb part of the shift, allowing earlier features to remain largely unchanged, improving \textcolor{lightblue}{\textbf{plasticity}}.}
\vspace{-0.6cm}
\label{fig:teaser}
\end{wrapfigure}
Recent works show that utilizing flow-matching can be highly effective for value  estimation in off-policy reinforcement learning (RL)~\citep{agrawalla2025floqtrainingcriticsflowmatching,espinosa2025expressive,dong2025valueflows,zhong2025flowcritic,chen2025unleashing}. These flow-matching critics depart from standard ``monolithic'' architectures that map state-action pairs to scalar Q-values in a single forward pass, by instead estimating values via the iterative integration of a learned velocity field given a noise input. This approach yields substantial empirical gains, and is more robust than monolithic networks in offline~\citep{levine2020offline} and offline-to-online RL settings~\citep{nakamoto2023calql}. However, while the performance gap is clear, the mechanism driving it is unexplained. Do flow-matching critics succeed because they model value distributions, or because they introduce an inductive bias that mitigates some pathologies of TD learning?

A natural hypothesis is that flow-matching critics succeed because they implicitly model return distributions, similar to distributional RL~\citep{bellemare2017distributional}. Indeed, much recent work applying flow matching to value learning~\citep{espinosa2025expressive,dong2025valueflows,zhong2025flowcritic,chen2025unleashing} explicitly adopts a distributional RL perspective and derives objectives accordingly. We design controlled experiments to study the role of distributional RL~\citep{bellemare2017distributional,dabney2017distributional}  and find that it does not explain the observed gains: explicitly incorporating distributional updates often degrades performance relative to simple ``expected-value'' backups, which do not learn the return distribution. Flow-matching critics trained with standard TD learning (i.e., expected Q-value backups) also consistently outperform strong distributional RL algorithms. These results indicate that the advantage of flow matching does not come from modeling return distributions.

\emph{\textbf{So why do flow-matching critics work?}} We argue that flow-matching critics improve TD learning because they train a velocity field with dense supervision covering an integration trajectory used for inference (Figure~\ref{fig:teaser}). This form of \emph{iterative computation} plays a dual role. First, at inference, the integration process enables the critic to refine its prediction as long as the velocity field is \emph{densely} supervised to fit the target for multiple interpolant values and through the integration trajectory. We term this phenomenon \textbf{\emph{test-time recovery}}, a phenomenon that arises because iterative computation on a densely supervised velocity field is used to estimate the Q-value. Second, beyond test-time recovery, flow-matching induces more \textbf{\emph{plastic}} features that allow better fitting subsequent TD targets. Specifically, the velocity network’s features need not change substantially in attempts to fit successive TD targets, as these shifts can instead be accommodated (at least partially) through the integration process. This mitigates pathologies associated with plasticity loss~\citep{lyle2023understanding} or representational capacity~\citep{kumar2021dr3}. Put together, flow critics outperform monolithic architectures with identical computation graphs because of no dense velocity supervision.

\textbf{Theoretically}, we formalize the notions of \emph{test-time recovery} and the preservation of \emph{plasticity} in simple settings, and show that flow matching induces weight update dynamics that preserve and reweight previously learned features rather than overwriting them under non-stationary TD targets. Such dynamics are absent in monolithic critics or their ensembles that do not use such iterative computation architectures. Achieving similar behavior in these baseline approaches instead requires auxiliary objectives or explicit regularization that bias optimization and must be carefully tuned. \textbf{Empirically}, we show that flow-matching critics tolerate substantially higher noise, are more robust to interventions such as resetting or freezing network components, and learn more isotropic features, thereby addressing several pathologies that have been reported with noise and regularization in TD targets. We also find that flow-matching critics achieve much stronger performance, resulting in a 2$\times$ performance gain and a 5$\times$ improvement in sample efficiency in high update-to-data online RL with offline data, a setting that is often plagued by loss of plasticity. Flow-matching critics also exhibit more stable learning at high UTD values. Finally, we show that these gains arise from fitting velocities rather than directly regressing to TD targets.

Our main contribution is an understanding of why flow matching improves TD learning. We show that value estimation with an iterative integration procedure, and training the velocity field with dense supervision along the integration path, fundamentally alters both inference-time behavior and representation learning. In particular, this structure enables test-time recovery and promotes more plastic feature representations. Together, these effects mitigate the core pathologies associated with TD learning.

\vspace{-0.2cm}
\section{Related Work}
\label{sec:related}
\vspace{-0.1cm}
\textbf{Flow matching in RL.}
Flow matching~\citep{flow_albergo2023,flow_lipman2023} and diffusion models~\citep{ho2020denoising,diffusion_sohl2015} are used as policies in RL~\citep{ren2024diffusion, celik2025dime, ma2025efficient, lv2025flow, wang2022diffusion, park2025flow}, motivated primarily from the lens of expressivity~\citep{diffusionpolicy_chi2023} but also more recently through iterative computation~\citep{pan2025adonoisingdispellingmyths}. Flow matching has been applied to value learning by parameterizing critic networks as velocity fields in TD learning~\citep{agrawalla2025floqtrainingcriticsflowmatching, espinosa2025expressive, dong2025valueflows,zhong2025flowcritic,chen2025unleashing}. While prior work shows strong results, it is unclear whether these gains stem from capacity, distributional modeling~\citep{espinosa2025expressive, dong2025valueflows,zhong2025flowcritic,chen2025unleashing}, or the learning dynamics induced by iterative computation~\citep{agrawalla2025floqtrainingcriticsflowmatching}. We provide evidence for the latter, showing that iterative computation from flow-matching produces more plastic features and improves value prediction.

\textbf{Stability and plasticity in TD learning.}
Prior work has identified several pathologies in TD learning, including value overestimation~\citep{hassalt10doubleq, fujimoto2018addressing}, growth in parameter norms~\citep{nikishin2022primacy, Nauman2024overestimation}, and rapid loss of plasticity~\citep{nikishin2023deep, lyle2023understanding} as training progresses. These issues are commonly attributed to bootstrapping and target non-stationarity, motivating a range of stabilization techniques, including architectural modifications such as layer normalization~\citep{ln_ba2016, ball2023efficient, nauman2024bigger}, weight or feature normalization~\citep{kumar2021dr3, hussing2024dissecting, lee2025hyperspherical}, and alternative objectives such as categorical losses~\citep{bellemare2017distributional,kumar2023offline,farebrother2024stop}. Collectively, these approaches improve representational capacity and help maintain plasticity~\citep{kumar2021implicit,kumar2021dr3,lyle2024normalization}, and are widely used in modern RL algorithms~\citep{lee2025hyperspherical, nauman2025bigger, palenicek2025xqc}. 

We show that flow-matching critics stabilize TD learning by improving plasticity and robustness by training parameters along with an integration procedure, without requiring additional forms of regularization. For instance, one of our experiments shows that freezing layers of a critic network typically cripples a standard critic in an offline RL setting (note that the closest prior results on recovering from feature resets~\citep{nikishin2022primacy} operate in an online RL setting where new data can be collected during the process of recovery). In contrast, we show that this type of intervention is far less detrimental for flow-matching critics that use an integration process and train the velocity with dense supervision along this process.

\textbf{Related work beyond TD learning.} Finally, our results relate to prior work analyzing the learning dynamics of flow-matching models outside of RL~\citep{li2025back,bertrand2025closedformflowmatchinggeneralization}. In contrast to these works, we study flow matching under non-stationary TD targets and show that this distinction leads to different conclusions and highlights the advantages of flow matching in RL. Our work is also related to \citet{pan2025adonoisingdispellingmyths}, which emphasizes the role of iterative computation in training policies via behavioral cloning on a static dataset and finds that two integration steps are sufficient in that setting. However, unlike TD learning, their setting does not involve clear non-stationarity or an explicit need for preserving plastic feature representations.

\vspace{-0.25cm}
\section{Preliminaries, Notation, and Setup}
\label{sec:prelims}
\vspace{-0.2cm}
Under usual notation for states $\bs$ and actions $\ba$, we train a policy $\pi(\ba| \bs)$ in a Markov decision process (MDP) that induces a distribution over the return  $Z^\pi(\bs, \ba) \triangleq \sum_{t=0}^\infty \gamma^t r(\bs_t, \ba_t)$. The expectation of $Z^\pi(\bs, \ba)$ is the Q-function of the policy: $Q^\pi(\bs, \ba) = \mathbb{E}[Z^\pi(\bs, \ba)]$. We focus on RL training with an offline dataset $\mathcal{D}=\{(\bs, \ba, r, \bs')\}$, in the standard offline RL~\citep{levine2020offline} setting for most analysis. We also consider the online RL training setting with offline data, where we use the offline dataset $\mathcal{D}$ to seed the replay buffer. Value-based methods learn a network $Q_\theta(\bs, \ba)$ by minimizing TD error.

\textbf{Flow-matching critics.} Flow-matching value functions depart from monolithic architectures that directly map $(\bs,\ba)$ to the value output after a few layers. Instead, they represent values via a learned transformation of a random noise $\bz \in \mathbb{R}$. Concretely, these methods parameterize a time-dependent velocity field
$v_\theta(\bz, t \mid \bs,\ba)$ that defines an ODE over $\bz$.
Starting from an initial noise sample $\bz_0 \sim p_0(\bz)$ at $t=0$, numerical integration of this ODE ($\psi(t, \bz|\bs, \ba)$) 
produces a value sample at $t=1$.
Several recent works use this parameterization to learn \emph{distributional} value functions.
In particular, \citet{espinosa2025expressive,dong2025valueflows,zhong2025flowcritic,chen2025unleashing} train the velocity field so that integrating over the initial noise recovers the full return distribution $Z^\pi(\bs,\ba)$.
Given a transition $(\bs,\ba,r,\bs') \sim \mathcal{D}$, a distributional TD target is constructed by pushing forward noise through the target flow at the next state:
{\setlength{\abovedisplayskip}{6pt}
 \setlength{\belowdisplayskip}{6pt}
\begin{align}
\label{eq:dist_target}
\!\!\!\bz' \sim p_0(\bz),~
\tilde{Z}(\bs,\ba; \bz') = r(\bs,\ba) + \gamma \, \psi_{\bar{\theta}}(1, \bz'|\bs', \ba'),\!\!~~~ (\bar{\theta}~\text{denotes target network})
\end{align}}
\!\!where $\psi_{\bar{\theta}}$ denotes the integrated target flow, $\ba' \sim \pi(\cdot|\bs')$.
Flow matching is then applied to align the velocity field with the transport from $\bz$ (distinct from $\bz'$) to samples from $\tilde{Z}(\bs,\ba; \bz')$.
An objective~\citep{espinosa2025expressive} is
{{\setlength{\abovedisplayskip}{6pt}
 \setlength{\belowdisplayskip}{6pt}
\begin{align}
\label{eq:distributional_floq}
\mathcal{L}_{\text{dist}}(\theta) := \mathbb{E}_{{(\bs,\ba,r,\bs') \sim \mathcal{D}, \bz,\bz',\, t \sim \text{Unif}(0,1)}}
\Big[
\big\|
v_\theta(\bz(t), t \mid \bs,\ba)
- \bar{s}_{\bz,\bz'}(\bs,\ba)
\big\|_2^2
\Big],
\end{align}}
\!\!where the target velocity is $\bar{s}_{\bz,\bz'}(\bs,\ba) := \tilde{Z}(\bs,\ba; \bz') - \bz$, and the interpolant is
$\bz(t) = t \cdot \bz + (1 - t) \cdot \tilde{Z}(\bs,\ba; \bz')$.

\textbf{\texttt{floq} (flow-matching critics with expected-value TD backups).}
In contrast, \citet{agrawalla2025floqtrainingcriticsflowmatching} use flow matching to represent a Q-function while targeting only the \emph{expected} return.
Although \texttt{floq} integrates noise into a Q-value ``sample'' and therefore produces stochastic outputs during inference, its TD target collapses the next-state flow to a scalar expectation. Consequently, this approach does not learn or enforce a return distribution.
Specifically, given $m$ i.i.d.\ noise samples $\{\bz'_j\}_{j=1}^m$, define
{\setlength{\abovedisplayskip}{6pt}
 \setlength{\belowdisplayskip}{6pt}
 \begin{align}
\label{eq:sclaar_target}
y(\bs,\ba)
:=
r(\bs,\ba)
+ \gamma \frac{1}{m} \sum_{j=1}^m \psi_{\bar{\theta}}(1, \bz'_j \mid \bs', \ba'),
\end{align}}
\!\!which estimates the expected-value TD target $r(\bs,\ba) + \gamma Q_{\bar{\theta}}(\bs',\ba')$.
Flow-matching is then used to regress from initial noise $\bz \sim \text{Unif}[l,u]$ to target $y(\bs,\ba)$ via:
\begin{align}
\label{eq:standard_floq}
\mathcal{L}_{\text{floq}}(\theta) := \mathbb{E}_{{(\bs,\ba,r,\bs') \sim \mathcal{D},~ \bz \sim \text{Unif}[l, u],~ t \sim \text{Unif}(0,1)}}
\Big[
\big\|
v_\theta(\bz(t), t \mid \bs,\ba)
- \big( y(\bs,\ba) - \bz \big)
\big\|_2^2
\Big].
\end{align}
where the interpolant is $\bz(t) = t\cdot\bz + (1 - t)\cdot y(\bs, \ba)$. 
Thus, unlike the distributional objective above, \texttt{floq} does not learn or enforce a distributional Bellman equation; it uses flow matching purely as a \emph{parameterization} of the expected Q-function. That is, the target used to train the velocity field is an expectation in Equation~\ref{eq:standard_floq} and can be viewed as the expectation of $s_{\bz,\bz'}(\bs,\ba)$ over $\bz'$ alone.

\citet{agrawalla2025floqtrainingcriticsflowmatching} motivates the use of flow-matching from the perspective of \emph{iterative computation}, rather than distributional RL and hence chooses to utilize expected-value backups. Rather than fitting a Q-function in a single pass, this work suggests that integration steps enable a gradual refinement of value estimates. While this explanation is appealing at a high level, it leaves open what iterative computation provides \emph{formally}. In particular, if the benefits of flow matching arise solely from iterative computation at inference time, it is unclear why monolithic architectures fail to exhibit similar behavior.

\underline{\emph{\textcolor{lightblue}{\textbf{Our goal.}}}}
Given these different design choices underlying flow-matching critics, our goal is to identify the mechanisms by which flow-matching Q-functions improve TD learning.
How does iterative integration interact with TD bootstrapping?
Why does \texttt{floq} improve performance despite targeting only expected values?
We argue that the answer lies not in distributional modeling, but in \emph{test-time recovery} and improved \emph{representational plasticity} enabled by the use of iterative computation both during training and inference in flow-based critics, which we analyze in the remainder of the paper.

\vspace{-0.2cm}
\section{Is Distributional RL Crucial for Success of Flow-Matching Critics?}
\label{sec:floq_and_dist_rl}
\vspace{-0.1cm}

\begin{wraptable}{r}{0.52\textwidth}
\vspace{-0.4cm}
\centering
\caption{
\label{table:floq_distributional_vs_expected}
\footnotesize
\emph{\textbf{Comparing expected-value (\textbf{E}) vs. distributional (\textbf{D})}} flow-matching critics on representative OGBench tasks.
Each entry reports \textbf{E / D}. While both variants learn similar expected Q-values, \textbf{D} produces higher-variance estimates (that more closely reflect variation in the return distribution) but does not outperform \textbf{E}.
}
\vspace{-0.2cm}
\resizebox{0.5\textwidth}{!}{
\begin{tabular}{l|c|cc}
\toprule
Env. & Success (\%) & $Q_\theta(\bs,\ba)$ & $\mathrm{Var}_\bz(Q)$ \\
\midrule
hmmaze-large   & \textbf{52} / 30 & $-180 / -170$ & $0.2 / \mathbf{4.5}$ \\
antmaze-giant  & \textbf{86} / 74 & $-190 / -200$ & $0.1 / \mathbf{0.7}$ \\
cube-double    & 72 / 72        & $-130 / -130$ & $1.1 / \mathbf{6.3}$ \\
hmmaze-medium  & 94 / 94        & $-170 / -170$ & $0.3 / \mathbf{2.3}$ \\
\bottomrule
\end{tabular}
}
\vspace{-0.7cm}
\end{wraptable}
We now explicitly test whether TD updates from distributional RL are necessary for strong performance of flow-matching critics.
To do so, we use the TD-update from \texttt{floq}~\citep{agrawalla2025floqtrainingcriticsflowmatching} and modify it to use a distributional backup (Equation~\ref{eq:distributional_floq}), which resembles the loss function used by \citet{espinosa2025expressive,zhong2025flowcritic}. 
\\\\
Importantly, we keep the velocity field architecture and hyperparameters (e.g., integration steps, initial noise range) identical across the two variants and only change the loss function from Equation~\ref{eq:standard_floq} to Equation~\ref{eq:distributional_floq}. We compare expected-value \texttt{floq} and its distributional counterpart on four representative OG-Bench tasks~\citep{ogbench_park2025}, along different axes: \textbf{(a)} the expected Q-value recovered on the offline dataset; \textbf{(b)} variance of the learned Q-value distribution (note that \texttt{floq} Q-functions still parameterize a distribution); and \textbf{(c)} the performance of the policy.

\textbf{Results.} Observe in Table \ref{table:floq_distributional_vs_expected}, that both expected and distributional variants recover nearly expected Q-values that are close to each other. However, the statistics of the learned Q-value distributions differ substantially. For instance, the standard deviation of the expected variant is significantly \emph{lower} than that of the distributional variant, and much lower than what one would expect for the distribution of returns on this environment. This behavior is consistent with the training objective. Since \texttt{floq} does not attempt to match the return distribution, and OG-Bench tasks are expected to exhibit high-variance, multi-modal returns~\citep{espinosa2025expressive,dong2025valueflows,zhong2025flowcritic,chen2025unleashing}, regressing to the expected target yields substantially lower-variance estimates. These results confirm that while both algorithms learn stochastic Q-value predictions, \texttt{floq} does not model the return distribution in a distributional RL sense since statistics in Table~\ref{table:floq_distributional_vs_expected} differ substantially.

Despite this, in Table \ref{table:floq_distributional_vs_expected} we see that the distributional variant offers no benefits in performance and is often worse than expected-value \texttt{floq}. In addition, \citet{agrawalla2025floqtrainingcriticsflowmatching} also shows that \texttt{floq} outperforms strong distributional RL baselines such as C51~\citep{bellemare2017distributional} and IQN~\citep{dabney2018implicit}. Taken together, these findings demonstrate that flow-matching critics can perform extremely well even in the absence of distributional RL training, ruling it out as an explanation. These results also motivate the use of \texttt{floq} for the rest of the analysis.

\begin{AIbox}{Takeaway: Distributional RL does not explain success of flow-matching critics.}
\begin{itemize}[itemsep=-2pt]
 \setlength{\leftskip}{-15pt}
    \item Standard \texttt{floq} outperforms its distributional variant although it does not fit the return distribution, as it learns lower variance Q-value distributions. 
\end{itemize}
\end{AIbox}

\vspace{-0.2cm}
\section{Flow Matching Enables Test-Time Recovery}
\label{sec:ttr}
\vspace{-0.1cm}

So far, we have seen that the empirical gains of \methodname{} are not explained by distributional RL. Why, then, does flow matching work well? In this section, we develop a mental model of how flow-matching Q-functions learn robust value estimates by leveraging iterative computation.

\textbf{Our main claim here} is that a flow-matching critic can \emph{correct} imperfect intermediate estimates produced during the integration of the flow. As more integration steps are performed, the final Q-value becomes less sensitive to errors made earlier in the integration process. We refer to this as \textcolor{lightblue}{\textbf{\emph{test-time recovery (TTR)}}} and this is a direct consequence of the iterative process of integration. \textbf{Crucially however}, TTR does not arise from integration at test time alone, but from the \emph{training procedure} used. Flow matching provides dense supervision of local velocity predictions for multiple interpolant values throughout the integration trajectory. This kind of supervision induces a correction mechanism that we quantify via a geometric condition (Definition~\ref{def:conic}). Monolithic critics are not trained to satisfy such a condition.  This perspective explains the two central design knobs of flow-matching critics: \emph{\textbf{iterative computation}} and \emph{\textbf{dense supervision}}. In the next section, we show that this training procedure also induces \emph{plastic} features, that put together with TTR explains the efficacy of flow-matching for TD learning. 

\vspace{-0.2cm}
\subsection{Formalizing Test-Time Recovery}
\label{subsec:ttr_definition}
\vspace{-0.1cm}

We begin by formalizing the notion of test-time recovery.
Intuitively, test-time recovery describes the ability of a flow-matching Q-function to compensate for errors or inconsistencies introduced during inference, using a \textbf{\emph{single set of trained parameters}}, by controlling the number of integration steps for computing Q-values. This mechanism is absent in monolithic critics that only query the network once.

We formalize this notion in Definition~\ref{def:ttr} by deriving a condition that perturbations (i.e., $\xi_k$) to velocity evaluations at a given step are progressively dampened along the integration trajectory. This means that errors induced at earlier integration steps can be corrected by later integration steps.

\vspace{-0.08in}
\definecolor{figurebackground}{HTML}{F9F8F3}
\begin{tcolorbox}[
  colback=figurebackground,
  boxrule=0pt,
  left=6pt,
  right=6pt,
  top=6pt,
  bottom=6pt
]
\begin{definition}[\emph{Test-Time Recovery}]
\label{def:ttr}
Fix a state-action pair $(\bs,\ba)$ and let $\psi_\theta(t,\bz | \bs,\ba)$ denote the flow interpolant obtained by integrating the velocity field $v_\theta$.
Consider a $K$-step integrator with step size $\eta = 1/K$ and discrete times $t_k = k/K$.
Let $\{\psi^k\}_{k=0}^K$ be the unperturbed trajectory defined by
\begin{align*}
\psi^{k+1} = \psi^k + \eta\, v_\theta(\psi^k, t_k| \bs,\ba),
\end{align*}

and let $\{\tilde{\psi}^k\}_{k=0}^K$ be a perturbed trajectory from the same initial noise $\bz$ defined as
\begin{align*}
\tilde{\psi}^{k+1}
:=
\tilde{\psi}^k
+
\eta\big(
v_\theta(\tilde{\psi}^k, t_k \mid \bs,\ba)
+
\xi_k
\big),
~~ \tilde{\psi}^0 = \psi^0 = \bz,
\end{align*}
where $\{\xi_k\}_{k=0}^{K-1}$ are \emph{arbitrary} perturbations incurred in  velocity evaluations. We say that a trained flow-matching critic exhibits \emph{test-time recovery} if the terminal error of the flow
$\Delta_K(\bz) := \tilde{\psi}^K - \psi^K$
satisfies:
$\|\Delta_K(\bz)\| \;\le\; \beta_K  \big[\max_{k = 0}^{K-1}\|\xi_k\|\big]$,
for a stability factor $\beta_K < 1$ that decreases with $K$.
\end{definition}
\end{tcolorbox}
When $\beta_K$ decreases with $K$ in the definition above, increasing the number of integration steps still enables contraction towards the unperturbed target. 
Our Theorem~\ref{thm:conic_implies_ttr_strong} (in Appendix~\ref{ttr_proofs}) shows that flow-matching critics satisfy $\beta_K \propto K^{-c'}$,
for some positive constant $c'$. This decay arises because flow matching fits a velocity field at every timestep of the integration process and across a broad range of interpolant inputs $\bz$. Doing so enforces locally corrective updates throughout the trajectory. This notion of corrective updates can be formalized via a \emph{conic condition} on the learned velocity field as follows.
\begin{tcolorbox}[
  colback=figurebackground,
  boxrule=0pt,
  left=6pt,
  right=6pt,
  top=6pt,
  bottom=6pt
]
\vspace{-0.1cm}
\begin{definition}[$c$-conic condition on a given velocity field; \textit{simplified}]
\label{def:conic}
Let $[l,u]$ be the initial noise range, and let $[l_1 (\bs,\ba), u_1(\bs,\ba)]$ be the range of final outputs for a given $(\bs, \ba)$ pair. Assume that $|u_1(\bs,\ba) - l_1(\bs,\ba)| < |u - l|$,
so that the final output range is strictly narrower than the initial noise interval. Let $K > 0$ be the number of integration steps and define the conic region
\begin{align*}
\mathcal{C}_{K}(\bs,\ba) 
:= \Big\{(\bz,t) ~\Big|~ 
(1-t) \cdot l + t\cdot l_1(\bs,\ba) \le \bz \le (1-t) \cdot u  + t\cdot u_1(\bs,\ba), \;
0 \le t \le 1 - \nicefrac{1}{K}
\Big\}.
\end{align*}
We say that the velocity field $v_{\theta^*}(\bz,t \mid \bs,\ba)$ is $c$-conic (for $0<c<1$) on $\mathcal{C}_{K}(\bs,\ba)$ if
\begin{align}
\label{eq:derivative_condition}
\frac{\partial v_{\theta^*}(\bz,t \mid \bs,\ba)}{\partial \bz}
\le -\frac{c}{1-t}, \quad \forall (\bz, t) \in \mathcal C_{K}(\bs, \ba).
\end{align}
\end{definition}
\end{tcolorbox}
Intuitively, modulo some boundary effects, the $c$-conic condition ensures that, integration trajectories starting from $\bz \sim \text{Unif}~[l, u]$ remain within the conic region $\mathcal{C}_K$, and contract toward the target value over time. In particular, Equation~\ref{eq:derivative_condition} controls the rate at which the distance between ``particles'' $(\bz,t)$ and $(\bz+\Delta \bz,t)$ shrinks as they evolve under the flow. The contraction rate scales as $(1-t)^{-1}$, and therefore becomes stronger as $t \to 1$. This scaling reflects the geometry of the supervision region: the admissible interval for $\bz$ decreases with $t$, forming a narrowing cone that contracts. To remain within this shrinking region, the velocity field must induce sufficiently strong contraction as $t$ increases. Thus, the conic condition enforces a form of corrective behavior over time.

Under the $c$-conic condition (and mild assumptions~\ref{assump:smooth} and \ref{assump:boundary_condition} on $v_{\theta^*}$), Theorem \ref{thm:conic_implies_ttr_strong} shows that $\beta_K$ satisfies $\beta_K \propto K^{-c'}$ for a positive constant $c' > 0$. As long as we train the velocity field to map a broad range of initial interpolants $\bz \sim \text{Unif}~[l,u]$ onto a fixed expected target Q-value, it is natural to expect it to learn a  ``funnel'' shape and condition to be satisfied in practice. This provides a mental model for why flow-matching critics benefit from multiple integration steps. In contrast, monolithic critics lack such intermediate supervision and are not explicitly trained to learn features that can enjoy the contraction condition. Hence, $\beta_K$ need not decrease with $K$, and TTR may not arise.

\vspace{-0.25cm}
\subsection{Experimental Study: Test-Time Recovery (TTR) in Practice}
\label{subsec:ttr_experiments}
\vspace{-0.2cm}

We now evaluate the ability of flow-matching critics to perform test-time recovery. To do so, we introduce controlled perturbations into the integration process and measure how sensitive the Q-value estimates are to these interventions. Our goal is to evaluate if supervising the velocity field densely is important for TTR. For comparison, we also construct analogous perturbations for monolithic critics.

\textcolor{lightblue}{\textbf{\emph{Experiment:} Injecting  
staleness.}}
We split the integration procedure into two phases. After training a flow-matching critic for $T = 250,000$ gradient steps, we evaluate a modified procedure in which the first $\kappa\%$ of the integration steps for flow matching ($\kappa \in \{0, 25, 50, 75, 100 \}$) are intentionally performed using a stale snapshot of the velocity field taken from the checkpoint at time $T$, while the remaining steps are completed using the current network parameters.
We compare this procedure to a baseline which runs all integration steps with current parameters ($\kappa=0$). If flow-matching exhibits test-time recovery, then later integration steps should be able to correct \emph{at least some} errors arising from staleness of the early steps of the integration trajectory, even when the parameters of the velocity field cannot be changed.

\begin{wraptable}{r}{0.48\textwidth}
\vspace{-0.4cm}
\centering
\small
\caption{
\footnotesize{\textbf{\emph{Effect of injecting staleness into early integration steps}} of the flow-matching critic.
Entries are shaded by performance degradation relative to the best value within each environment
as $\kappa=\{0,25,50,75,100\}\%$ increases from left to right. Note that a flow-matching critic \emph{can} still succeed in many cases when the first 25\% or even 50\% of the integration is done with a stale velocity field. On two out of three environments, using a stale velocity field improves performance slightly (best performance in each row is in \textcolor{lightblue}{\textbf{bold}}).}
}
\label{tab:staleness_singlecol_shaded}
\vspace{-0.2cm}
\resizebox{0.99\linewidth}{!}{
\begin{tabular}{lc}
\toprule
Env. (default task) & Success as stale fraction $\kappa$ increases \\
\midrule
{antsoccer-arena} &
$\shadebyYR{45}{50}{35} \rightarrow \textcolor{lightblue}{\mathbf{50}} \rightarrow
\shadebyYR{47}{50}{35} \rightarrow \shadebyYR{46}{50}{35} \rightarrow \shadebyYR{35}{50}{35}$ \\

{hmmaze-medium} &
$\textcolor{lightblue}{\mathbf{98}} \rightarrow \shadebyYR{43}{98}{39} \rightarrow
\shadebyYR{63}{98}{39} \rightarrow \shadebyYR{62}{98}{39} \rightarrow \shadebyYR{39}{98}{39}$ \\

{cube-double} &
$\shadebyYR{72}{84}{58} \rightarrow \textcolor{lightblue}{\mathbf{84}} \rightarrow
\shadebyYR{82}{84}{58} \rightarrow \shadebyYR{58}{84}{58} \rightarrow \shadebyYR{64}{84}{58}$ \\
\bottomrule
\end{tabular}
}
\vspace{-0.4cm}
\end{wraptable}
\textbf{Results.} As hypothesized, in Table~\ref{tab:staleness_singlecol_shaded} we observe empirical evidence of test-time recovery. On 2 of 3 environments, using stale velocity parameters for the first 25\% or even 50\% of the integration steps in the critic yields policies with higher success rates than using no stale velocities at all, indicating that flow-matching critics can recover from errors introduced early during integration. This behavior is not universal, as on the \texttt{humanoidmaze-medium} task any amount of staleness leads to substantial performance degradation. Nonetheless, these results do show that staleness at early steps may not be a problem for \texttt{floq} networks, at least in some amount.

In contrast, applying an analogous intervention to monolithic feed-forward critics consistently results in pronounced performance drops, as shown in Figure~\ref{fig:plasticity} where we freeze early layers of the network. We observe similar degradation for ResNet critics with residual connections and  transformer critics (Figure~\ref{fig:resnet_plasticity_fql}), suggesting that a computation graph resembling an integration process is insufficient to induce TTR. Instead, our results show that recovery requires training with a flow-matching loss itself. Finally, as we discuss next, flow matching also induces more robust feature representations, explaining why flow-matching critics remain resilient not only to stale velocities but also to stale internal features.

\vspace{-0.15cm}
\begin{AIbox}{Takeaway: Injecting staleness into earlier steps of integration}
\begin{itemize}[itemsep=-2pt]
 \setlength{\leftskip}{-15pt}
    \item Flow-matching critics can recover from staleness of early integration computations but an analogous mechanism to recover from staleness is absent in monolithic critics.
\end{itemize}
\end{AIbox}

\textcolor{lightblue}{\textbf{\emph{Experiment:} Robustness to noisy TD targets.}}
In the previous experiment, we perturbed test-time integration by freezing early integration steps while keeping training unchanged. In this experiment, we instead corrupt the \emph{training-time supervision}. Specifically, we add randomly sampled i.i.d.\ zero-mean noise to the TD \emph{velocity} targets used at each training step. Although the injected noise has zero mean, a critic may overfit to these corrupted targets and consequently learn biased Q-values. As a result, this corruption should degrade policy performance for any method. However, if flow-matching critics exhibit test-time recovery (TTR), they should be more robust to increasing noise magnitude. To test this, we inject noise into the TD velocity targets at each training step following the procedure described in Appendix~\ref{subsec:offline_rl_analysis_details}. For comparison, we train a monolithic Q-network with the same noisy TD targets as well.

\begin{wrapfigure}{r}{0.6\linewidth}
\centering
\vspace{-0.4cm}
\includegraphics[width=0.99\linewidth]{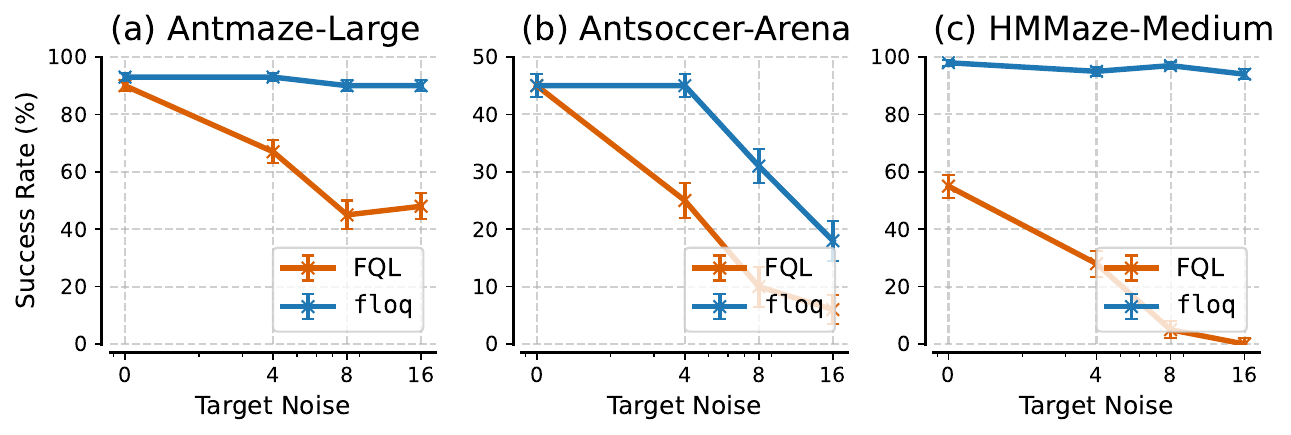}
\vspace{-0.2cm}
\caption{\footnotesize \textbf{\emph{Performance of flow-matching (\methodname{}) and monolithic (FQL) critics when trained with target noise}}. Observe that flow-matching critics are much more robust to noise in TD targets, while performance of FQL (monolithic critic) degrades substantially faster, even when they start at a similar point (antmaze/antsoccer).}
\vspace{-0.35cm}
\label{fig:target_noise_plots}
\end{wrapfigure}
\textbf{Results.}
Figure \ref{fig:target_noise_plots} shows that performance degradation is more graceful for the flow-matching critic, which consistently maintains higher performance  (and even suffers from no degradation in some cases) than the monolithic approach as noise magnitude increases. These results suggest that robustness to noisy supervision can arise in flow-matching critics. Overall, we see that the learned integration dynamics of flow critic remains more stable, allowing later integration steps to partially attenuate the effect of noisy supervision. 

\begin{AIbox}{Takeaway: Robustness to noisy TD targets from TTR}
\begin{itemize}[itemsep=-2pt]
 \setlength{\leftskip}{-15pt}
    \item While performance of monolithic networks degrades as we add i.i.d. noise into TD targets, flow critics are more robust and exhibit a smaller degradation for a given noise magnitude.
\end{itemize}
\end{AIbox}

\vspace{-0.2cm}
\section{Flow Matching Learns Plastic Features}
\label{sec:plasticity}
\vspace{-0.1cm}

The experiments in the previous section show that flow-matching critics exhibit test-time recovery, and that this behavior arises from dense supervision along the integration trajectory rather than from the computation graph alone. We now argue that this same training mechanism also induces more \emph{plastic} representations in the velocity network. By fitting velocities at each integration step, the network learns features that can support multiple TD targets when evaluated through the integration process at test time, without requiring the features themselves to be substantially modified. Such plastic features lead to improved stability and performance in high update-to-data (UTD) ratio settings.

\textbf{Intuition.} Non-stationary TD targets computed on unseen actions at the next state ($\ba' \sim \pi(\cdot|\bs')$) can lead to pathologies such as feature rank collapse, exploding feature norms, and dead neurons, as critics must repeatedly overwrite their features to track moving targets, eventually exhausting representational capacity~\citep{kumar2021implicit,lyle2022learning,nikishin2022primacy}. We hypothesize that flow-matching critics alleviate these issues by supervising the network densely through the integration process. Instead of directly modifying features to match each new TD target, this process distributes corrections to the value estimate across integration steps. As a result, the integration process itself can accommodate large changes in predicted values without requiring substantial modifications to internal representations. Consequently, learned features (especially in early layers) do not overfit to current targets and retain plasticity. \textcolor{lightblue}{\emph{\textbf{Viewed through this lens, training with a flow-matching loss plays a dual role:}}} \emph{at test time}, it enables recovery and improves robustness of Q-value estimates; \emph{at training time}, it buffers non-stationary TD targets from network representations, promoting more stable feature learning. We now present evidence supporting this hypothesis.

\vspace{-0.2cm}
\subsection{Theoretical Analysis in a Linear Setting}
\label{subsec:toy_setting}
\vspace{-0.1cm}

We start by analyzing a simple yet informative linear setting. We consider TD learning with a \emph{linear} flow-matching critic and compare it to a corresponding ``deep'' linear monolithic critic~\citep{arora2018optimization}. Although both models express linear functions of the input, their learning dynamics differ.

\textbf{Linear model setup.}
We consider learning linear predictors on inputs $\bx\in\mathbb{R}^d$, when training against non-stationary targets. We denote the non-stationary targets as $y(m)$, where $m$ denotes the training step. A \textbf{\emph{monolithic} critic} is given by
$f_{\mathrm{mono}}(\bx;m) \triangleq w(m)^\top \bx$, where $w(m)\in\mathbb{R}^d$ is the effective weight and $\bx$ compactly represents the concatenated input $[\bs,\ba]$. The effective weight vector $w(m)$ is given by a product of multiple linear layers, i.e., $w(m) = \prod_{t=1}^T u_t(m)$. A deep linear ResNet~\citep{he2015deep} is given by: $w(m) = \prod_t (I + u_t(m))$. Any monolithic critic directly trains the parameters in $w(m)$ against the non-stationary $y(m)$ by minimizing squared error via gradient descent.

A \textbf{\emph{flow-matching} critic} with a comparable linear architecture can be characterized explicitly by unrolling the integration process. Since the interpolant at each step depends on the output of the previous step, the resulting mapping is recursive in the input $\bx$. Consider a linear velocity network at each integration step. Let $\{u_t(m)\}_{t=1}^{T-1}$ denote the linear velocity slices and let $\{\alpha_t\}_{t=1}^{T-1}$ denote the integration step sizes. Then the output velocity field is: $v_\theta(\bz, t|\bx) = v_t(m) \cdot \bz + u_t(m)^\top \bx$. We refer to $u_t(m)$ as the weights parameterizing the $t$-th ``slice'' and $v_t(m)$ denotes the gain parameters that multiply the interpolant. 

Unrolling the $T$-step Euler integration yields the following \emph{expected} output of the integration process. Here, expectation is taken over the initial noise $\bz$. If the noise mean is non-zero, an additional bias term independent of $\bx$ appears, which we remove for clarity as it does not change any of our conclusions.
{
\setlength{\abovedisplayskip}{5pt}
\setlength{\belowdisplayskip}{5pt}
\begin{align}
f_{\mathrm{FM}}(\bx;m)
=
\sum_{t=1}^{T-1}
\beta_t(m) \cdot 
\underbrace{u_t(m)^\top \bx}_{\text{\textcolor{lightblue}{\textbf{Output from $t$-th linear ``slice''}}}},
\label{eq:fm_linear}
\end{align}
}\!where the integration-dependent amplification coefficient is $\beta_t(m)
\triangleq
\alpha_t
\prod_{k=t+1}^{T-1}
\big(1 + \alpha_k v_k(m)\big)$. 
Here $\{v_k(m)\}_{k=1}^{T-1}$ denote scalar ``gain'' parameters induced by recursive integration. The gains $v_k(m)$ quantify how the contribution of the $t$-th velocity computation propagates through subsequent integration steps to influence the final Q-value prediction. Because each integration step multiplicatively rescales the interpolant via a factor depending on $v_k(m)$, earlier slices can be amplified or attenuated depending on downstream gain dynamics, leading to a structured reweighting of feature directions. Thus, although $f_{\mathrm{FM}}$ is linear in $\bx$, its effective weight vector
\begin{align}
w_{\mathrm{FM}}(m)
=
\sum_{t=1}^{T-1}
\beta_t(m)\,u_t(m)
\end{align}
is not freely parameterized. Instead, it decomposes into velocity slices whose coefficients are coupled through the integration procedure. \textbf{\textit{We argue below that}} monolithic networks can adapt only by updating $w(m)$ (and hence, each of $u_t(m)$ since there is one supervision on the entire weight $w(m)$), whereas flow-matching critics can adapt by reweighting existing features $\{u_t(m)\}$ through changes in the gain dynamics $\{v_t(m)\}$ due to dense supervision. Our theoretical result for this setting is given below.

\vspace{-0.1cm}
\begin{tcolorbox}[
  colback=figurebackground,
  boxrule=0pt,
  left=6pt,
  right=6pt,
  top=6pt,
  bottom=6pt
]
\begin{theorem}[\textbf{Flow critics can learn by reweighting existing features; monolithic critics must modify features.}]
\label{thm:feature_reweighting}
Consider training monolithic and flow-matching models by minimizing squared error against a non-stationary target $y(m)$. Fix an interval of training steps $m \in [m_0,m_1]$ and suppose
\begin{align*}
\dot u_t(m)=0
\quad
\text{for all } t \text{ and } m\in[m_0,m_1],
\end{align*}
i.e., feature directions are frozen during this training interval. Then the following hold:

\begin{enumerate}[itemsep=1pt]
\item \textbf{(Monolithic).}
If $\nicefrac{d}{dm} f_{\mathrm{mono}}(\cdot;m) \neq 0$ {on } $m \in [m_0,m_1]$, then necessarily $\exists u_t ~\text{s.t.}~ \dot u_t(m)\neq 0$.

\item \textbf{(Flow-matching).}
The Euler flow-matching predictor satisfies
\begin{align*}
\frac{d}{dm} f_{\mathrm{FM}}(\bx;m)
=
\Big(
\sum_{t=1}^{T-1}
\dot\beta_t(m)\,u_t(m)
\Big)^\top \bx,
\end{align*}
where
\begin{align*}
\dot\beta_t(m)
=
\beta_t(m)
\sum_{k=t+1}^{T-1}
\frac{\alpha_k\,\dot v_k(m)}{1+\alpha_k v_k(m)}.
\end{align*}
Thus, even when feature directions $\{u_t(m)\}$ remain \textbf{all} fixed, the predictor can adapt through changes in the gain parameters $\{v_t(m)\}$, i.e., $\dot v_k(m) \neq 0 \;\Rightarrow\;  \dot\beta_t(m) \neq 0,\;\forall\, t<k$.

\iffalse
\item \textbf{(Different responses at different integration steps).}
With squared loss, the gain dynamics satisfy
\begin{align*}
\dot v_t(m)
=
-2\,\alpha_t\,\big(f_{\mathrm{FM}}(\bx;m)-y(m)\big)\,s_t(\bx;m),
\end{align*}
where $s_t(\bx;m)$ denotes the slice-specific activation entering the $t$-th integration step. Since $s_t$ contains a factor proportional to $(1-\alpha_t)$ multiplying the target, the sensitivity of $\dot v_t(m)$ to changes in $y(m)$ scales with $(1-\alpha_t)$. Consequently, slices with smaller $\alpha_t$ (typically later in integration) respond more strongly to changes in the target, while earlier slices respond more weakly.
\fi
\end{enumerate}
\end{theorem}
\end{tcolorbox}
% \vspace{-0.1cm}

\textbf{Theorem interpretation.}
This theorem isolates the mechanism underlying feature learning in monolithic and flow critics. 
When the feature directions $\{u_t(m)\}$ are held fixed over an interval $[m_0, m_1]$ of gradient updates, a flow-matching critic can still track changes in the regression target by adapting the predictor exclusively through the coefficients $\{\beta_t(m)\}$. These coefficients evolve through the gain dynamics $\{\dot v_t(m)\}$ (Appendix~\ref{app:toy_setting}), which arise from the recursive structure of integration at test time and the dense supervision during training. Thus, the effective weight vector $w_{\mathrm{FM}}(m) = \sum_{t=1}^{T-1} \beta_t(m)\,u_t(m)$
can change even when the feature directions of the network (i.e., $u_1, u_2, \cdots, u_T$) themselves remain fixed. 
Adaptation to the non-stationary target is therefore mediated by a redistribution of contributions across integration slices, allowing the predictor to reallocate mass among previously learned features without overwriting them.

\textcolor{lightblue}{\emph{\textbf{We remark that the distinction here \underline{does not} arise from representation capacity of the network.}}} The difference lies in dense vs final output supervision: flow matching provides step-wise supervision at each step of integration, inducing gain dynamics that enable feature reweighting. In contrast, a monolithic predictor $f_{\mathrm{mono}}(\bx; m)= w(m)^\top \bx$ that trains the final output to match the target admits no such reweighting behavior. Thus, changing the monolithic critic requires $\dot w(m)\neq 0$, and hence, at least one of $u_1, u_2, \cdots, u_T$ must change. Thus monolithic critics need to modify features to track  the current TD target, while flow-matching critics can learn fairly ``general'' features that can represent multiple TD targets over the course of learning. This formalizes our intuition about feature preservation and plasticity discussed at the beginning of this section. For completeness, we extend Theorem~\ref{thm:feature_reweighting} to ensembles of monolithic networks in Appendix~\ref{sec:monolithic_linear_new}, and identify a similar limitation with them.

\vspace{-0.2cm}
\subsection{Empirical Evidence of Plastic Feature Learning}
\label{subsec:plasticity_experiments}
\vspace{-0.1cm}

To empirically test this intuition, we examine whether flow-matching critics learn more plastic and qualita-tively different feature representations than monolithic networks. In particular, we ask whether the interac-
\begin{wrapfigure}{r}{0.66\textwidth}
\vspace{-0.1cm}
\centering
\includegraphics[width=0.99\linewidth]{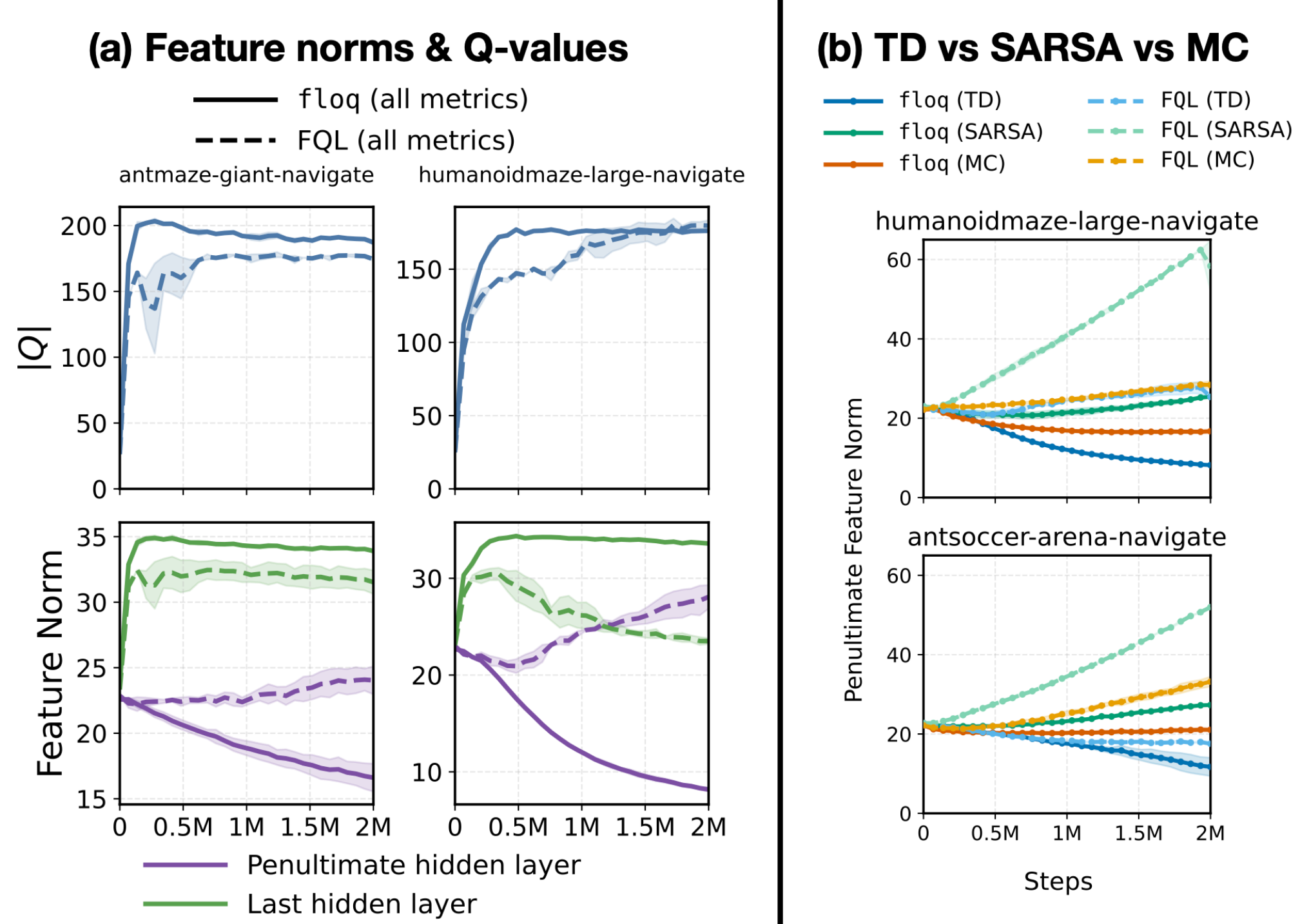}
\vspace{-0.1cm}
\caption{\footnotesize \textbf{\emph{Feature norms.}}
\textbf{(a)} Learned feature norms and average Q-values for monolithic critics (FQL) and flow-matching critics (\methodname{}) in the penultimate and last hidden layers. While the last hidden layer adapts to the scale of Q-values for both methods, the penultimate hidden layer in \methodname{} exhibits a much more rapid decrease in feature norms compared to FQL. This indicates that \methodname{} learns more flexible and adaptive features in the penultimate hidden layer that are largely decoupled from the magnitude of Q-values.
\textbf{(b)} Penultimate hidden layer feature norms for \methodname{} trained with TD, SARSA, and MC targets. \methodname{} with TD shows the fastest decrease in feature norms, whereas SARSA and MC trends resemble those of the monolithic FQL critic. This suggests that flow-matching critics, particularly under TD learning, develop more robust representations under non-stationary targets.}
\vspace{-0.4cm}
\label{fig:floq_fql_feature_norm}
\end{wrapfigure} 
tion between flow matching and TD learning leads to differences in learned representations relative to monolithic critics, and whether such differences are specific to non-stationarity in TD learning. We run several experiments that we describe in the section below.

\textcolor{lightblue}{\textbf{\emph{Experiment:} Measuring properties of learned features.}}
We measure the $\ell_2$-norm of \emph{post-layernorm} features learned by the velocity network for flow-matching critics and the critic network for monolithic critics for three algorithms:
\textbf{(a)} TD learning,
\textbf{(b)} SARSA (i.e., using the dataset action for the TD backup), and
\textbf{(c)} supervised regression to pre-computed Monte Carlo (MC) returns.
As shown in Figure \ref{fig:floq_fql_feature_norm} (ref. Fig \ref{fig:floq_fql_feature_norm_appendix} for more tasks), across several tasks, flow-matching critics trained with TD learning learn much lower-norm penultimate hidden-layer features than monolithic critic networks, despite the absence of any explicit regularization, which has in fact appeared as a desirable property~\citep{kumar2023offline, hussing2024dissecting}.
This also indicates that a significant burden of modeling the Q-value scale is deferred to the final layers and the integration, rather than encoded through the network, despite not putting any explicit regularizer to do so. As such, the model is less likely to learn spurious features to explain the changes in magnitude of TD targets.

Notably, this effect is absent when training with SARSA or Monte-Carlo (MC) regression (Figure~\ref{fig:floq_fql_feature_norm}; see also Figure~\ref{fig:td_mc_sarsa_appendix} for additional tasks), where both flow-matching and monolithic critics exhibit similar feature statistics. This suggests that the representational differences arise specifically from the interaction between flow matching and TD learning, particularly the non-stationarity induced by bootstrapped TD targets during training. In contrast, SARSA and MC regression admit more stationary targets, especially later in training. As a result, both models only need to learn features that ``overfit'' to a fixed value function, and differences in features of flow-matching and monolithic models wear off over training.

\begin{AIbox}{Takeaway: Feature norms at all but last layers reduce for flow critics}
\begin{itemize}[itemsep=-2pt]
 \setlength{\leftskip}{-15pt}
    \item The norm of features at all but last layers of the network reduces with more training for flow-matching critics, while it increases for standard critics. These trends are absent for SARSA and supervision regression onto Monte-Carlo return run for enough training steps.
\end{itemize}
\end{AIbox}

\begin{figure}[t] 
\centering 
\vspace{-0.2cm} 
\includegraphics[width=0.95\linewidth]{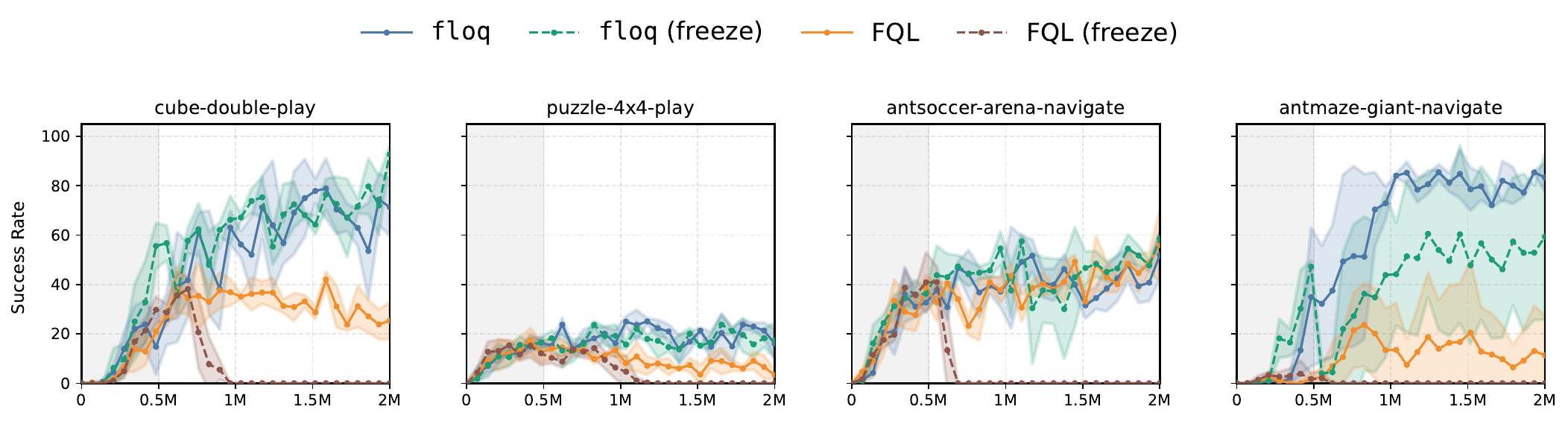} \vspace{-0.4cm} 
\caption{\footnotesize \emph{\textbf{Measuring feature plasticity}} on four tasks, by freezing all layers except the final two at $T = 0.5M$ steps (gray shaded region denotes the pre-freeze phase). Solid curves correspond to the default (fully trained) runs, while dashed curves show performance after freezing the penultimate hidden features. Across all environments, \emph{FQL with a monolithic critic exhibits a sharp performance collapse} once features are frozen (brown vs orange), indicating in inability to represent future Q-functions. In contrast, flow-matching critics remain stable and continue to improve after freezing, showing substantially greater plasticity.} 
\vspace{-0.3cm} 
\label{fig:plasticity} 
\end{figure}

\begin{wrapfigure}{r}{0.73\textwidth}
\vspace{-0.15cm}
\centering
\includegraphics[width=0.99\linewidth]{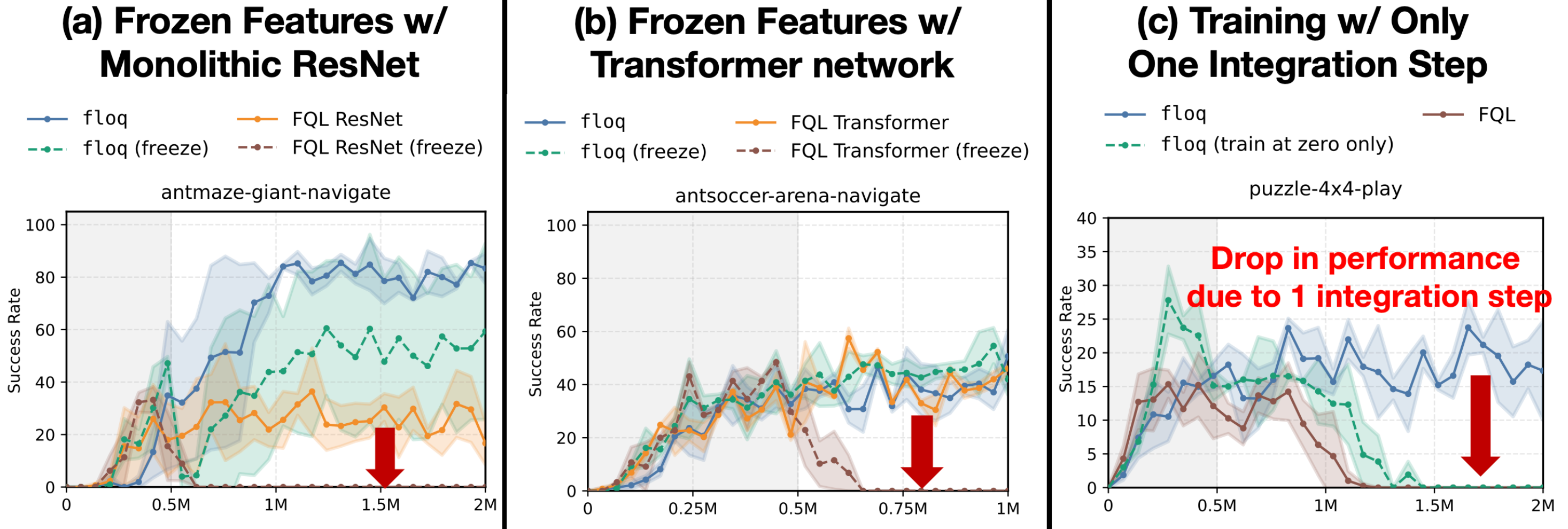}
\vspace{-0.1cm}
\caption{\footnotesize
\emph{\textbf{(a) Frozen features with a monolithic ResNet critic.}}
Although a monolithic ResNet admits a computation graph similar to a flow-matching critic, freezing its features leads to a collapse in performance during subsequent offline RL training. \emph{\textbf{(b) Frozen features with a monolithic transformer-based critic.}} Despite performing well with the standard FQL algorithm, monolithic transformer critics still suffer a performance collapse when their layers are frozen during subsequent training, indicating that the plasticity issue persists across model architectures.
\emph{\textbf{(c) Frozen features with a single integration step.}}
With only one integration step, a flow-matching critic is more stable than a monolithic network, but less stable than full flow matching with multiple integration steps (performance drop indicated by the \textcolor{red}{red} arrow), highlighting the essential role of integration in preserving feature plasticity.
}
\vspace{-0.3cm}
\label{fig:resnet_plasticity_fql}
\end{wrapfigure}
\textcolor{lightblue}{\textbf{\emph{Experiment:} Measuring feature plasticity by freezing features.}} To assess whether these representational differences are consequential, we probe feature plasticity by freezing the early layers of the critic at an intermediate point during offline training and continuing TD learning on the offline dataset. If the learned features are sufficiently expressive to support future TD targets, the impact of freezing should diminish with further training. Observe in Figure~\ref{fig:plasticity}, freezing features causes monolithic critics (both ResNet and transformer-based\footnote{Our transformer architecture represents the state and action jointly using four tokens, which are fed into a bi-directional transformer. We found that this design achieves performance comparable to, and usually slightly better than, a ResNet-based FQL critic across a variety of tasks. Before adopting this simpler transformer architecture, we also experimented with the transformer design with an attention entropy regularizer proposed in TQL~\citep{dong2026tql}, anticipating that it might better mitigate plasticity loss. However, across the harder OG-Bench tasks we evaluated, the official TQL implementation consistently underperformed our simpler transformer design with no regularizer (e.g., 15\% for TQL on antsoccer-arena vs. 40\% for ours, and 0\% for TQL on humanoid-maze-large vs. 45\% for ours). Hence, we use our transformer (with no regularizer) here in the interest of simplicity.} architectures) to collapse to near-zero performance across almost all environments, with little to no recovery. In contrast, flow-matching critics recover to performance comparable to the unfrozen baseline on all but one environment, where performance remains competitive. This indicates that features learned by flow-matching critics remain useful for representing \emph{future value functions} that will be encountered, even when these features are no longer updated. This supports the notion that flow-matching learns features that help model the value improvement path~\citep{dabney2020value}, and hence learn plastic features that do not inhibit or slow-down future value learning. We further confirm this behavior by applying the same intervention to a monolithic critic with a ResNet-style architecture (see Figure~\ref{fig:resnet_plasticity_fql}), which also collapses once intermediate layers are frozen. 

Finally, to understand the impact of training critics with dense supervision on an integration path, we additionally evaluate a variant of flow-matching critics  trained with only a single integration step (``train at zero only''), where no integration is performed for training or inference. While this variant outperforms a monolithic critic (FQL) after the point of intervention (Figure~\ref{fig:resnet_plasticity_fql}; Figure~\ref{fig:plasticity_train_at_zero_only_appendix} for more tasks), it is substantially less stable than full flow matching, highlighting the crucial role of integration in enabling plastic representations. This means that training via a flow-matching loss is important but supervising at all interpolants through time imbues the network with better representations. 
Finally, we note that unlike in online RL, where learning can recover from frozen/reset features by collecting new data online~\citep{nikishin2022primacy}, such mechanisms are unavailable in offline RL. This means the learned features must remain suitable for supporting future TD targets in order to maintain performance in the setting of this experiment.

\begin{AIbox}{Takeaway: Features learned by flow-matching better represent future TD targets}
\begin{itemize}[itemsep=-2pt]
 \setlength{\leftskip}{-15pt}
    \item Flow-matching critics learn representations that can support future TD targets better: freezing features severely degrades monolithic critics but has little effect on flow-matching critics.
\end{itemize}
\end{AIbox}

\begin{wrapfigure}{r}{0.45\textwidth}
\centering
\vspace{-0.4cm}
\includegraphics[width=0.99\linewidth]{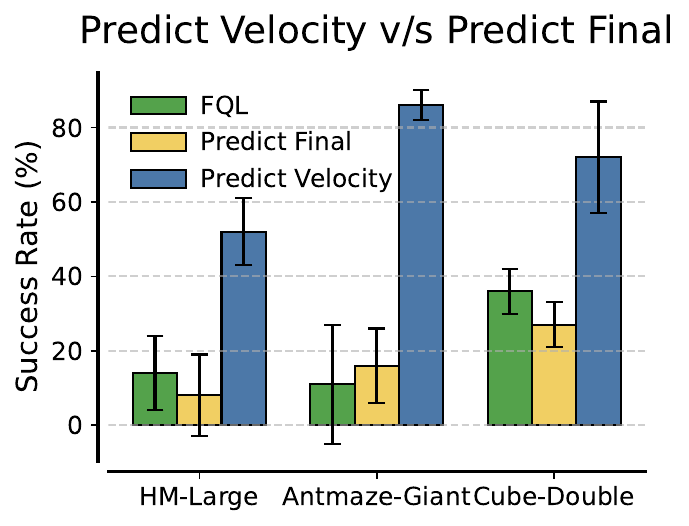}
\vspace{-0.8cm}
\caption{\footnotesize \textbf{\emph{Impact of predicting velocities vs. final TD targets}} at each integration step. Predicting the final TD-target substantially degrades the performance of a flow-matching critic, yielding performance comparable to monolithic critics. This underscores the importance of dense velocity supervision in flow-matching critics.}
\vspace{-0.5cm}
\label{fig:predict_local}
\end{wrapfigure}
\textcolor{lightblue}{\textbf{\emph{Experiment:} Intermediate velocity supervision is crucial.}} Finally, we test whether the specific choice of supervising the \emph{velocity field}, rather than absolute TD targets, is crucial for obtaining benefits of flow matching discussed so far.
Motivated by recent work~\citep{li2025back}, we implement a variant of flow matching in which each integration step is supervised to directly predict the TD target value itself, rather than the velocity.  As shown in Figure \ref{fig:predict_local}, although this variant still uses integration and receives dense supervision at every integration step, it fails to retain the benefits of flow matching degrading to a performance comparable to that of a monolithic critic. Empirically, we find that the network learns to ignore the interpolant and instead fit the target values independently at each step, effectively collapsing to an ensemble of monolithic critics.
As a result, performance degrades and the gains from iterative computation disappear. This behavior highlights that training a velocity field at each interpolant is critical for attaining better performance.

\begin{AIbox}{Takeaway: Fitting velocity is important for flow-matching critics.}
\begin{itemize}[itemsep=-2pt]
 \setlength{\leftskip}{-15pt}
    \item Supervising \emph{velocities} is essential for these benefits; directly supervising absolute TD targets collapses flow matching to monolithic behavior, eliminating test-time recovery and plasticity.
\end{itemize}
\end{AIbox}

\vspace{-0.2cm}
\section{Application: Flow Matching Critics Enable High-UTD Online RL with Prior Data}
\label{sec:app}
\vspace{-0.1cm}

The analysis above shows that the gains of flow matching (TTR and plastic features) arise from training the velocity field densely along the integration trajectory. Here, we examine whether these properties improve performance in high update-to-data (UTD) online learning regimes, where standard RL is believed to suffer from several pathologies such as plasticity loss~\citep{nikishin2022primacy}, value overestimation~\citep{redq}, etc. If a flow-matching critic indeed learns features capable of representing and fitting non-stationary TD targets, its performance should scale faster when increasing UTD, and not destabilize at very high values of UTD. To test this, we incorporate flow-matching critics into the RLPD framework~\citep{ball2023efficient}, enabling more aggressive data reuse. However, doing so required some modifications to the RLPD algorithm.

\begin{figure}[h]
    \centering
    \includegraphics[width=0.99\textwidth]{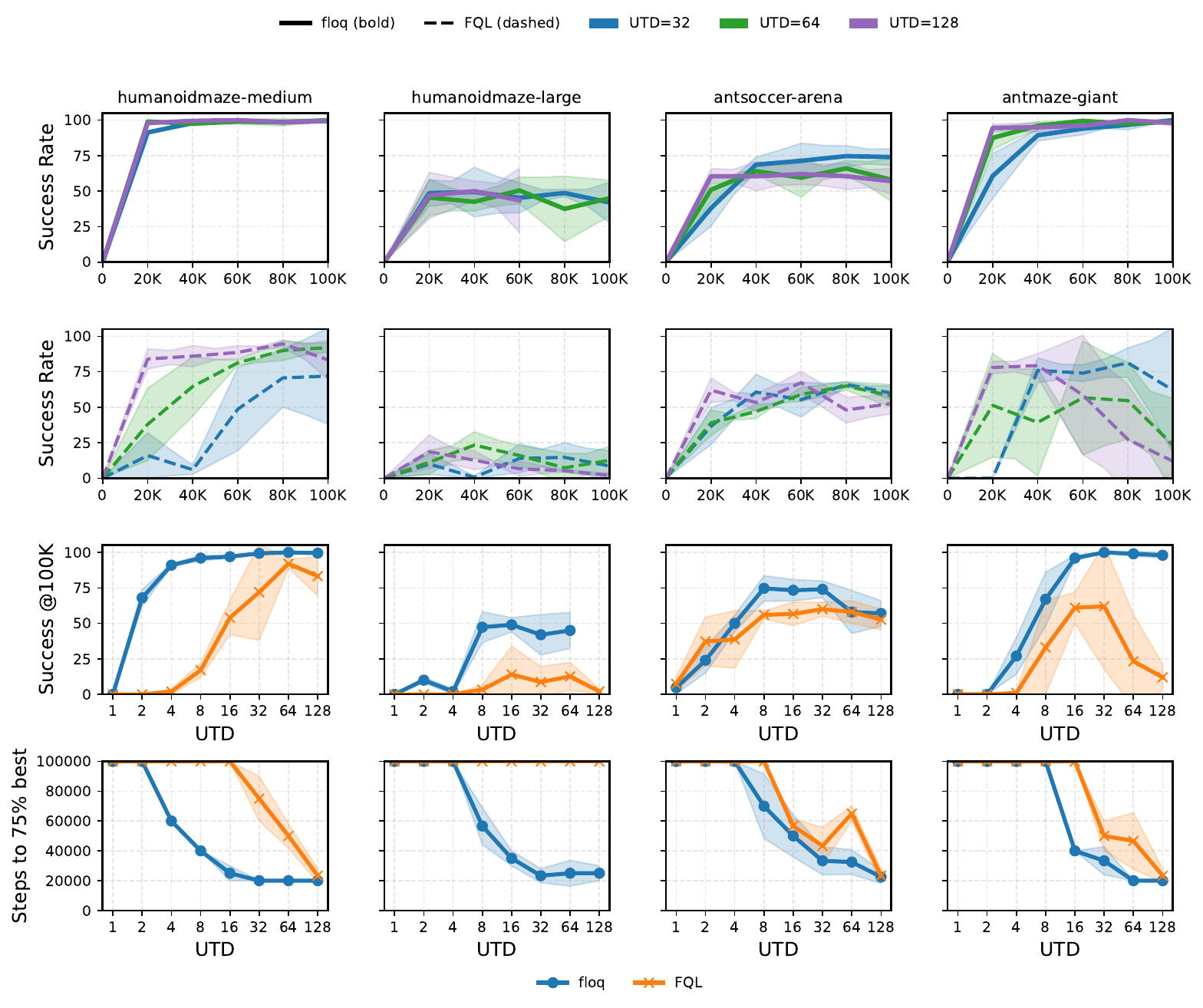}
    \vspace{-0.4cm}
    \caption{\footnotesize
    \textbf{Comparing monolithic and flow-matching critics with RLPD on four hard OG-bench environments.}
    Each column corresponds to a different environment 
    (humanoidmaze-medium, humanoidmaze-large, antsoccer-arena, antmaze-giant).
    \textbf{Row 1:} Performance curves for \texttt{floq} at UTD $\in \{32,64,128\}$.
    \textbf{Row 2:} Performance curves for FQL at UTD $\in \{32,64,128\}$.
    \textbf{Row 3:} Final success rate at $200$K environment steps versus UTD.
    \textbf{Row 4:} Sample efficiency measured as the number of environment steps required
    to reach $75\%$ of the best achieved final performance (capped at a maximum of 100,000).
    Shaded regions indicate standard deviation across random seeds. Observe that \texttt{floq} outperforms monolithic FQL on success rates at high UTD values (row 3) by 2$\times$ on many tasks, stability (rows 1 vs 2), and efficiency (row 4) by 5$\times$ on many tasks.
    }
    \label{fig:rlpd_utd_analysis}
    \vspace{-0.1cm}
\end{figure}

\textbf{Adapting RLPD to use flow policies.} To ensure a fair and competitive comparison, we build on RLPD~\citep{ball2023efficient} and adapt it for our setting of training flow policies needed to make any non-trivial progress on OG-Bench tasks. These modifications substantially improve the performance of both the monolithic baseline and flow-matching critics. First, since we replace the Gaussian policy with a flow-matching policy, we train it using an FQL-style objective. To stabilize training, we incorporate a behavior cloning (BC) regularizer, which we find necessary when using flow policies but not with Gaussian ones. With flow policies and BC regularization, we found that enforcing a strict 50:50 mixture between offline and online data was unnecessary; instead, we use a standard replay buffer initialized with the offline dataset without subscribing to particular mixing proportions. We also omit a ReDQ~\citep{redq}-style large ensemble, as we do not observe significant overestimation in our settings even with high UTD values. Thus, our final recipe is to initialize the replay buffer with offline data, apply BC regularization to the policy, and scale the UTD ratio. Importantly, we apply the same modifications to the non-flow baseline (denoted ``FQL''), as well as the tested flow-matching critic model, ensuring that improvements are attributable to the flow-matching critic rather than to tuning differences. We refer to Appendix~\ref{sec:addl} \& \ref{app:hyperparameters}  for more details.

\textbf{Tuning the BC coefficient $\alpha$.} Because we initialize the replay buffer with offline data and focus on high-UTD training, careful tuning of the behavior cloning (BC) coefficient $\alpha$ is important. For both methods, we tune $\alpha$ over uniformly spaced values in the range $[\alpha_{\text{offline}} - 20, \alpha_{\text{offline}} + 20]$,
where $\alpha_{\text{offline}}$ denotes the best-performing BC coefficient in the purely offline RL setting reported by \citet{agrawalla2025floqtrainingcriticsflowmatching}. We apply the same tuning procedure to baseline FQL for a fair comparison. All  other hyper-parameters are kept to the same values as the offline to online RL setting in \cite{agrawalla2025floqtrainingcriticsflowmatching}, for both \methodname{} and FQL. 

\textbf{Results.} As shown in Figure~\ref{fig:rlpd_utd_analysis}, high-UTD training with \methodname{}, a flow-matching critic, achieves substantially stronger performance at large UTD ratios, with roughly 2$\times$ higher final return and a 5$\times$ improvement in sample efficiency compared to an RLPD baseline on top of FQL, using monolithic critics. We also find that flow-matching critics are more stable and do not destabilize, even at the highest UTD value we tested. These results show the efficacy of flow-matching critics in high-UTD online RL settings.

\vspace{-0.2cm}
\section{Discussion and Perspectives on Future Work}
\label{app:discussion}
\vspace{-0.1cm}

We explain the efficacy of flow-matching critics in off-policy RL by identifying dense supervision over iteratively allocated compute as the key mechanism. Rather than benefiting from distributional modeling, these critics jointly learn a velocity field and an integration procedure, which induces test-time recovery and preserves plastic representations under non-stationary TD targets. This coupling between training and iterative computation enables more robust adaptation to noise, target drift, and feature interventions, helping mitigate common pathologies in high-UTD settings.

\textbf{Future work.} On the practical side, it would be natural to explore alternative parameterizations of the velocity field,  for example, higher-dimensional velocity representations or structured gain mechanisms, and to evaluate their effectiveness in settings that place stronger demands on plasticity, such as continual or lifelong RL with shifting task distributions. From a theoretical perspective, extending our analysis beyond the linear setting to nonlinear function approximation remains an important open problem. More broadly, our results highlight a connection between plasticity under non-stationarity, flow-matching dynamics, and iterative computation. While we study these interactions in the context of TD learning, the underlying principles may extend to other domains where models must repeatedly adapt to evolving targets, such as time-series modeling or adaptive control. 

Next, we discuss a direction for future investigation that connects our findings with language models.

\textbf{Connections to LLM reasoning.} The role of integration steps in flow matching closely parallels the role of reasoning steps in large language models (LLMs) trained to generate chain-of-thought (CoT). Increasing the number of integration steps corresponds to allocating additional test-time compute~\citep{snell2024scaling}, allowing the model to iteratively refine its prediction through a sequence of intermediate states. Crucially, in both settings, iterative computation is beneficial only when it is aligned with the training objective. In flow matching, dense supervision along intermediate interpolants ensures that each step learns a meaningful local correction toward the TD target. Similarly, in LLMs, scaling test-time compute improves only when models are trained via RL to use intermediate reasoning productively~\citep{setlur2025scalingtesttimecomputeverification}.

Expanding on this, each state-action pair induces a new ``task'' at each training step that corresponds to fitting the corresponding TD target, equivalent to a new prompt in the LLM case. A monolithic critic must learn to solve all such tasks through a single static mapping from state-action pairs to Q-values, and must adapt this mapping over training steps. As our experiments showcase, repeatedly fitting evolving targets with a static mapping  leads to plasticity loss since features can overfit to each target. In contrast, flow matching learns an iterative procedure: instead of directly mapping inputs to values, the model learns how to progressively transform an initial estimate toward the target. This effectively equips the network with a reusable operator that can adapt its prediction by allocating additional compute as needed within the budget of $T$ integration steps, rather than by immediately changing its internal representations. A similar phenomenon appears in LLM reasoning: reasoning LLMs that learn to spend more tokens on a prompt than needed generalize better especially on task mixtures consisting of heterogeneous difficulty problems, where any single static predictor does not effectively generalize~\citep{setlur2025opt}.

More broadly, these connections suggest a broader principle underlying our approach tying to bi-level optimization or meta learning. Weight updates provide a slow ``outer-loop'' mechanism for adapting representations across training steps, while iterative computation at inference offers a fast ``inner-loop'' mechanism for adapting predictions within a fixed parameterization. How to best parameterize both the inner and outer loops for best use of compute is an open question. In flow-matching critics, dense supervision along the integration trajectory serves as a surrogate objective that trains the network’s weights to effectively utilize this fast adaptation mechanism. In LLMs, mid-training followed by RL plays an analogous role, since it builds in priors that allow models to chain basic ``skills''~\citep{setlur2025e3learningexploreenables}. But what is the right parameterization for critics and how to spend variable computation are open questions for critics. Formalizing these connections may lead to improved methods across domains.

\vspace{-0.1cm}
\section*{Acknowledgements}
\vspace{-0.1cm}
We thank Zheyuan Hu, Amrith Setlur, Lehong Wu, Max Sobol Mark, Sreyas Venkatraman, Preston Fu, and Paul Zhou for feedback on an earlier version of this paper. We thank all members of the CMU AIRe lab for their support. AK thanks Amrith Setlur, Abhishek Gupta, and Max Simchowitz for helpful discussions. This work is supported by the Office of Naval Research under N00014-24-2206, Schmidt Sciences AI2050 Early Career Fellowship, and an Amazon gift. We thank the TRC program of Google Cloud for TPU resources, babel cluster at CMU, and NCSA Delta for providing GPU resources that enabled this work. This work used the Delta advanced computing and data resource at the National Center for Supercomputing Applications (NCSA) through allocation CIS250548 from the Advanced Cyberinfrastructure Coordination Ecosystem: Services And Support (ACCESS) program.
Delta is supported by the National Science Foundation (award OAC-2005572) and the State of Illinois, and ACCESS is supported by U.S. National Science Foundation grants 2138259, 2138286, 2138307, 2137603, and 2138296.
\bibliography{main}
\newpage

\appendix
\onecolumn

\part*{Appendices}

\section{Additional Experimental Results}
\label{sec:addl}

\textbf{Post-layernorm feature norms for flow-matching critics (\methodname{}) vs monolithic critics (FQL).} We show in Figure \ref{fig:floq_fql_feature_norm_appendix} that while the last hidden layer adapts to the scale of Q-values for both \methodname{} and monolithic critics, the penultimate hidden layer in \methodname{} exhibits a rapid decrease in feature norms compared to FQL. This indicates that \methodname{} learns more adaptive representations in the penultimate hidden layer that are largely decoupled from the Q-value scale.

\textbf{Penultimate hidden layer feature norms comparison for \methodname{} (TD), \methodname{} (SARSA), and \methodname{} (MC).}
Figure~\ref{fig:td_mc_sarsa_appendix} shows that \methodname{} trained with TD learning exhibits the fastest decrease in penultimate hidden layer feature norms, whereas the trends for \methodname{} (SARSA) and \methodname{} (MC) closely resemble those of the monolithic FQL critics. This suggests that flow-matching critics develop more robust representations specifically under non-stationary TD targets. In contrast, with relatively stationary objectives such as regression to Monte-Carlo returns or SARSA, where Q-values are evaluated only for the behavior policy, the targets change less over time. As a result, the model has greater opportunity to absorb the target scale directly into the learned features, reducing resemblance to features learned in TD learning.

\textbf{Frozen features with a single integration step.} We show in Figure \ref{fig:plasticity_train_at_zero_only_appendix} that flow critics with one integration step are more performant and stable than monolithic critics once features are frozen, but fall short in learning features that can support subsequent TD updates than full flow matching critics. This highlights that integration plays a crucial role in preserving feature plasticity. 
\begin{figure*}[h]
\centering
\vspace{0.2cm}
\includegraphics[width=0.99\textwidth]{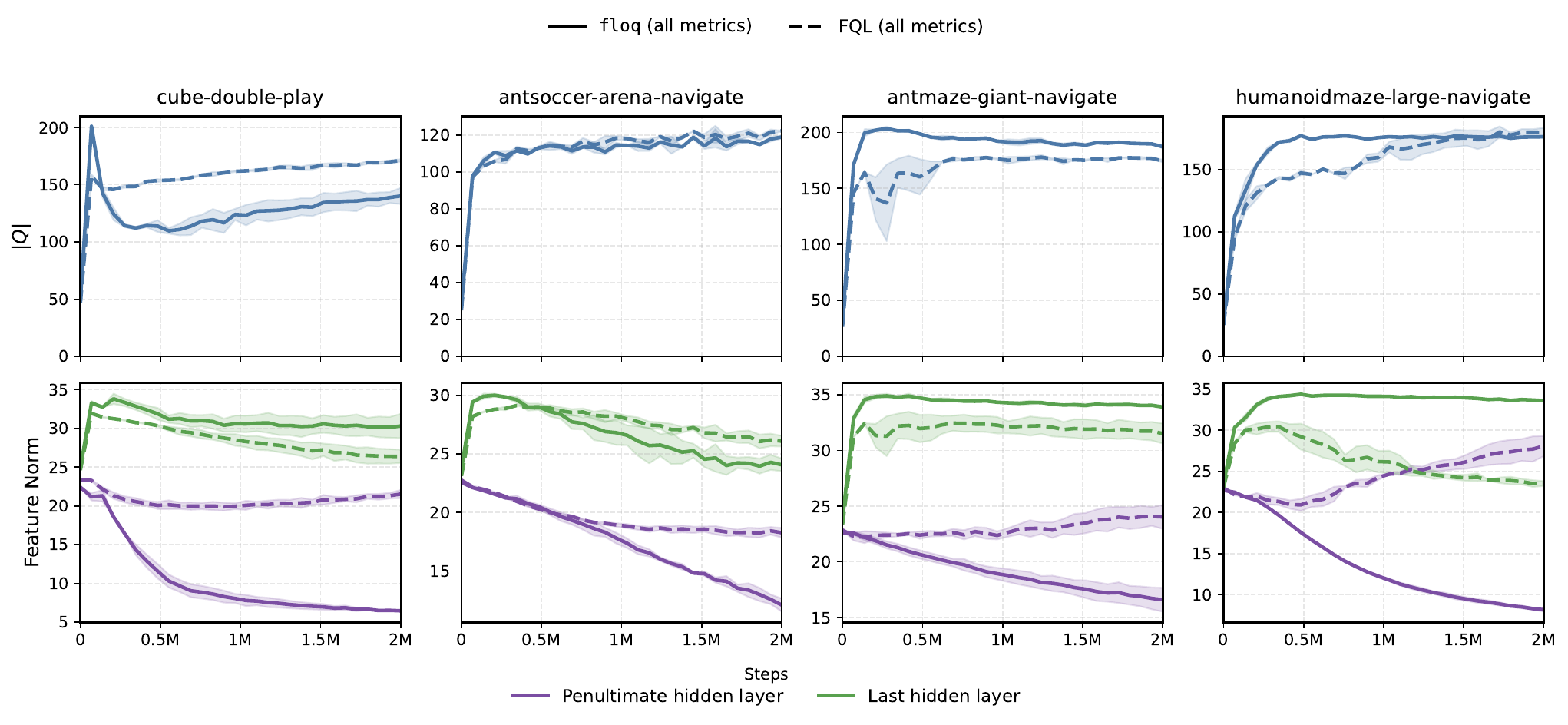}
\vspace{-0.2cm}
\caption{\footnotesize While the last hidden layer adapts to the Q-value scale for both \methodname{} and monolithic critics, the penultimate hidden layer layer in \methodname{} exhibits a qualitatively distinct, rapid decrease in feature norms. This indicates that \methodname{} learns more adaptive representations in the penultimate hidden layer that are largely decoupled from the Q-value scale.}
\vspace{-0.3cm}
\label{fig:floq_fql_feature_norm_appendix}
\end{figure*}

\begin{figure*}[ht]
\centering
\vspace{0.2cm}
\includegraphics[width=0.99\textwidth]{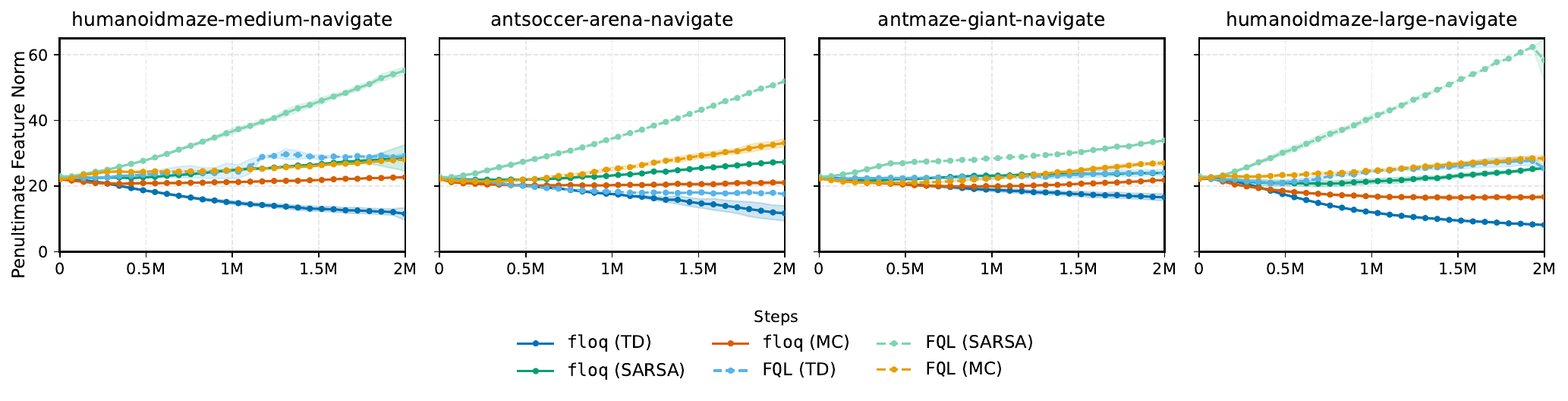}
\vspace{-0.2cm}
\caption{\footnotesize Flow-matching critics trained with TD learning (\methodname{}) shows the fastest decrease in penultimate hidden layer feature norms, whereas \methodname{} (SARSA) and \methodname{} (MC) trends resemble those of the monolithic FQL critics. Thus flow-matching critics develop particularly better representations under non-stationary TD targets.}
\vspace{-0.3cm}
\label{fig:td_mc_sarsa_appendix}
\end{figure*}

\begin{figure*}[t]
\centering
\vspace{0.2cm}
\includegraphics[width=0.99\textwidth]{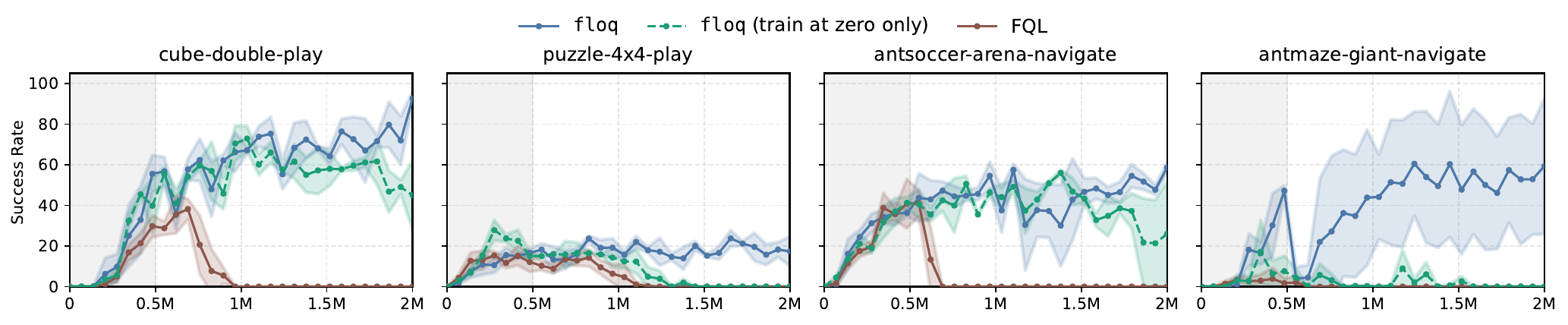}
\vspace{-0.2cm}
\caption{\footnotesize  Flow-matching critics with one integration step are more stable than monolithic critics but less stable than full flow matching. This emphasizes that test-time recovery in the form of integration plays a crucial role in preserving feature plasticity.}
\vspace{-0.1cm}
\label{fig:plasticity_train_at_zero_only_appendix}
\end{figure*}

\textbf{Complete results for RLPD (Section~\ref{sec:app}) with flow-matching and monolithic critics.} Figure~\ref{fig:rlpd_all_utd_grid} shows performance across all update-to-data (UTD) ratios and environments. At very low UTD (e.g., UTD=1), neither method makes meaningful progress. As UTD increases, however, flow matching (\methodname{}) consistently achieves stronger performance. In particular, for UTD values greater than or equal to 8, flow-matching critics outperform monolithic critics in both sample efficiency and final performance across most environments. At very high UTD values (e.g., 64 or 128), monolithic critics occasionally match the final performance of flow matching, but their learning dynamics are generally less stable. Overall, these results corroborate the benefits of flow-matching critics over monolithic ones.

\section{Benchmarks}
\label{app:benchmarks}
Following evaluation protocols from recent work in offline RL~\citep{park2025flow,wagenmaker2025steering,espinosa2025scaling, espinosa2025expressive,dong2025valueflows}, we use the \mbox{\textbf{OGBench}} task suite~\citep{ogbench_park2025} as our main evaluation benchmark (see Figure \ref{fig:ogbench_tasks}).
OGBench provides a number of diverse, challenging tasks across robotic locomotion and manipulation,
where these tasks are generally more challenging than standard D4RL tasks~\citep{d4rl_fu2020}, which have been saturated as of 2024~\citep{rebrac_tarasov2023, d5rl_rafailov2024, park2024value}.
While OGBench was originally designed for benchmarking offline goal-conditioned RL,
we use its reward-based single-task variants (``\texttt{-singletask}'' from \citet{park2025flow}).

We used the \texttt{default} task in each OGBench environment for our analysis and RLPD experiments, following the protocol in recent works ~\citep{park2025flow,wagenmaker2025steering,espinosa2025scaling, espinosa2025expressive,dong2025valueflows}.

\begin{figure}[htbp]
    \centering
    \vspace{-0.2cm}
    \includegraphics[width=0.99\textwidth]{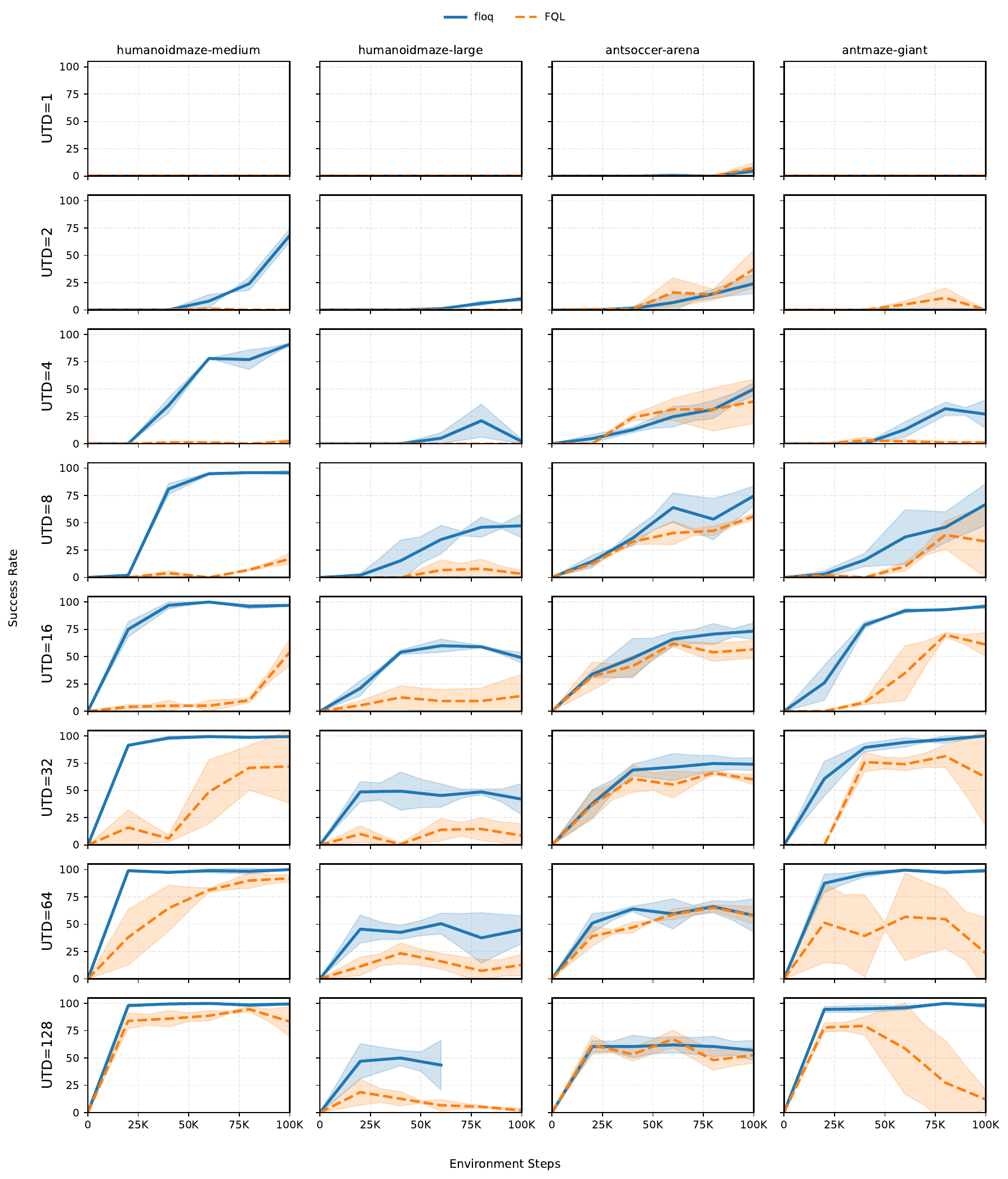}
    \vspace{-0.3cm}
    \caption{ \footnotesize
    \textbf{Performance profiles across all update-to-data (UTD) ratios.}
    Each column corresponds to an environment and each row corresponds to a UTD ratio.
    Solid lines denote \texttt{floq} and dashed lines denote FQL.
    Curves show mean success rate over environment steps and shaded regions
    indicate standard deviation across random seeds.
    }
    \label{fig:rlpd_all_utd_grid}
\end{figure}

\begin{figure*}[ht]
\centering
\vspace{-0.2cm}
\includegraphics[width=0.7\textwidth]{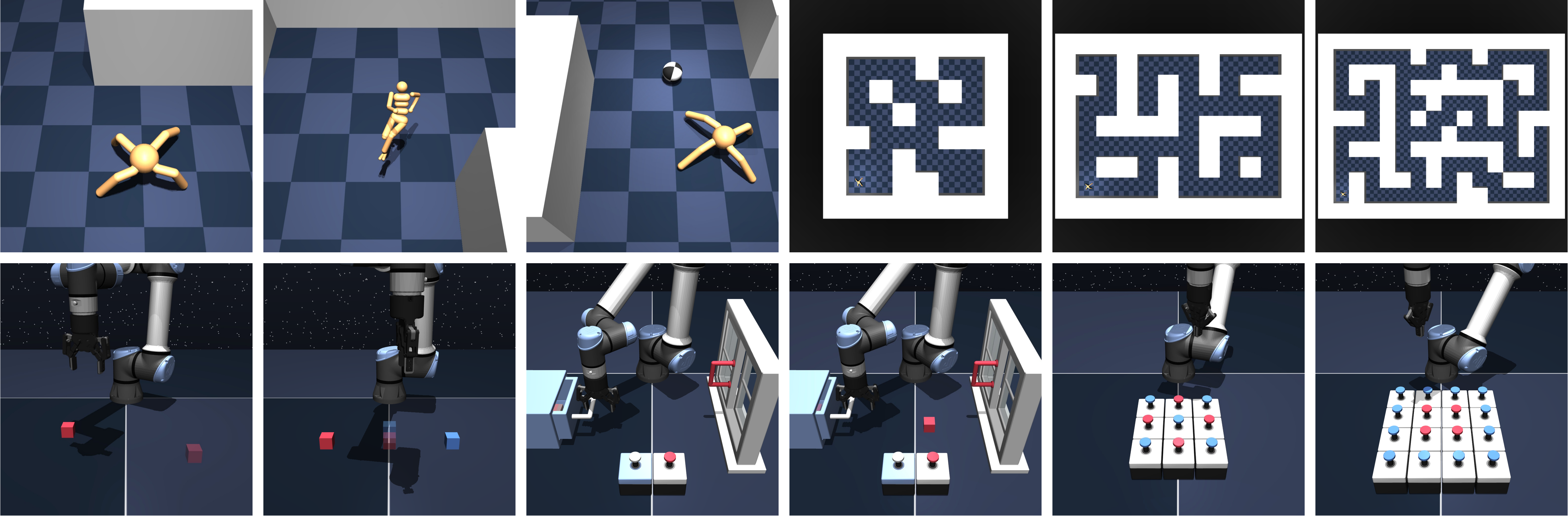}
\vspace{-0.2cm}
\caption{\footnotesize 
 \textbf{OGBench~\citep{ogbench_park2025} domains}. These tasks include high-dimensional state and action spaces, sparse rewards, stochasticity, as well as hierarchical structure.} 
\label{fig:ogbench_tasks}
\vspace{-0.1cm}
\end{figure*}

\section{Experimental Details and Hyperparameters}
\label{app:hyperparameters}

\subsection{Offline RL Analysis Experiments (Section \ref{sec:floq_and_dist_rl}, Section \ref{sec:ttr}, Section \ref{sec:plasticity}).}\label{subsec:offline_rl_analysis_details}

\textbf{Experimental details for robustness to noisy TD supervision (Figure \ref{fig:target_noise_plots}, Section \ref{subsec:ttr_experiments}).} We added $\text{Unif}[-\kappa, \kappa]$ (for $\kappa \in \{0, 4, 8, 16\}$) noise to the TD targets for FQL (monolithic critics) and to the TD \emph{velocity} targets for \methodname{} (flow-matching critics) at each training time-step.

\textbf{Hyperparameters.} We used the hyper-parameters from \citet{agrawalla2025floqtrainingcriticsflowmatching} for both \methodname{} (flow-matching critics) and FQL (monolithic critics)  for all offline RL analysis experiments. 

\vspace{-0.2cm}
\section{Theoretical Results for Test-Time Recovery (Section \ref{subsec:ttr_definition})}
\vspace{-0.2cm}
\label{ttr_proofs}
In this section we show that the $c$-conic condition implies the test-time recovery guarantee stated in the main paper. Specifically, we prove that the stability factor $\beta_K$ in Definition~\ref{def:ttr} decays polynomially with the number of integration steps $K$ when the conic condition is satisfied. Throughout this section we fix a state-action pair $(\bs,\ba)$ and omit it from notation when clear.

\vspace{-0.2cm}
\subsection{Assumptions}
\vspace{-0.2cm}

We impose mild smoothness and boundary assumptions on the learned velocity field.

\begin{tcolorbox}[
  colback=figurebackground,
  boxrule=0pt,
  left=6pt,
  right=6pt,
  top=6pt,
  bottom=6pt
]
\begin{assumption}[Smoothness of the velocity field]
\label{assump:smooth}
The velocity field $v_{\theta^*}(\bz,t)$ is continuously differentiable in $\bz$ and satisfies a local Lipschitz condition:
\begin{align*}
\|v_{\theta^*}(\bz_1,t)-v_{\theta^*}(\bz_2,t)\|
\le \frac{L}{1-t} \|\bz_1-\bz_2\|,
\end{align*}
for all $(\bz,t)$ in the conic region $\mathcal C_K$, for some constant $L > 0$.
\end{assumption}
\end{tcolorbox}

This assumption is standard for neural velocity fields and holds for common architectures with smooth activations. 

Recall the conic region $\mathcal C_{K}(\bs, \ba)$ from Definition ~\ref{def:conic}. For any $\varepsilon_g > 0$, define it's \emph{boundary} regions
\begin{align}
\mathcal{L}_{K, \varepsilon_g}(\bs,\ba) 
:= \Big\{(\bz,t) \Big| 
(1-t) \cdot l + t\cdot l_1(\bs,\ba) \le \bz \le (1-t) \cdot [l + \varepsilon_g]  + t\cdot l_1(\bs,\ba), \;
0 \le t \le 1 - \nicefrac{1}{K}
\Big\},
\end{align}
and
\begin{align}
\!\!\!\!\mathcal{U}_{K, \varepsilon_g}(\bs,\ba) 
:= \Big\{(\bz,t) \Big|
(1-t) \cdot [u - \varepsilon_g] + t\cdot u_1(\bs,\ba) \le \bz \le (1-t) \cdot u  + t\cdot u_1(\bs,\ba), \;
0 \le t \le 1 - \nicefrac{1}{K}
\Big\}.
\end{align}

\begin{tcolorbox}[
  colback=figurebackground,
  boxrule=0pt,
  left=6pt,
  right=6pt,
  top=6pt,
  bottom=6pt
]
\begin{assumption}[Boundary conditions]
\label{assump:boundary_condition}
There exists $\varepsilon_g, \delta_g >  0$ such that 
\begin{align}
v_{\theta^*}(\bz, t | \bs, \ba) &> l_1(\bs, \ba) - l + \delta_g \quad \forall (\bz, t) \in \mathcal L_{K, \varepsilon_g}(\bs, \ba) ~~~~\text{and}~~~~\\
v_{\theta^*}(\bz, t | \bs, \ba) &< u_1(\bs, \ba) - u - \delta_g \quad \forall (\bz, t) \in \mathcal U_{K, \varepsilon_g}(\bs, \ba)
\end{align}
\end{assumption}
\end{tcolorbox}
This assumption enforces that the velocity field points "inwards" near the boundary, which prevents the integration trajectories from escaping the supervision region $\mathcal C_{K}(\bs, \ba)$. Most learned velocity fields will be trained to point inward overall as long as the training loss is not too high, but this condition is needed for our analysis to discard ``adversarial'' velocity fields that might be smooth but may take on adversarial values on specifically chosen points near the boundary, disrupting our analysis. Moreover, such a boundary condition could be enforced in principle as a hard constraint by clipping the network's predictions during training.  

\vspace{-0.2cm}
\subsection{Trajectory Containment in the Supervision Region}
\label{sec:trajectory_containment}
\vspace{-0.2cm}
Our test-time recovery guarantees rely on the $c$-conic condition holding along the inference-time integration trajectory. Since this condition is defined only on the supervision region $\mathcal C_K(\bs, \ba)$, we first establish that the learned dynamics remain inside this region.

\textbf{Intuitively,} the supervision region shrinks over time toward the terminal output interval. Because the learned velocity field transports mass from a broad initial noise distribution to a strictly narrower target distribution, the induced dynamics tend to point inward near the boundary of the region, preventing trajectories from escaping except with small probability due to perturbations. The following statement formalizes this intuition by showing that the conic region is forward-invariant with high probability under the integration dynamics. This containment result ensures that the $c$-conic condition applies throughout the integration trajectory, which will enable the contraction analysis and test-time recovery guarantees established in the subsequent section.

\begin{tcolorbox}[
  colback=figurebackground,
  boxrule=0pt,
  left=6pt,
  right=6pt,
  top=6pt,
  bottom=6pt
]
\begin{lemma}[Trajectory containment]
\label{thm:trajectory_containment}
Assume that the learned velocity field satisfies Assumption~\ref{assump:smooth} (velocity smoothness), Assumption~\ref{assump:boundary_condition} (boundary condition), and the $c$-conic condition from Definition~\ref{def:conic}. Let $\{\tilde\psi^k\}_{k=0}^K$ be the perturbed Euler trajectory:
\begin{align}
\tilde\psi^{k+1}
=
\tilde\psi^k
+
\eta \big(v_{\theta^*}(\tilde\psi^k,t_k)+\xi_k\big),
\quad
\eta=\frac{1}{K}, 
\quad t_k=\frac{k}{K},
\end{align}
with $\tilde\psi^0 \in [l,u]$. Then there exists $\bar\xi > 0$ such that if $\max_{0 \le k \le K-1} \|\xi_k\| \le \bar\xi$, then we have
\begin{align*}
(\tilde\psi^k,t_k) \in \mathcal C_K
\quad \text{for all } k=0,\dots,K.
\end{align*}
\end{lemma}
\end{tcolorbox}

\begin{proof}
Define the time-step dependent boundaries
\begin{align*}
b_L(t) &= (1-t) \cdot l + t \cdot l_1, \\
b_U(t) &= (1-t) \cdot u + t \cdot u_1.
\end{align*}

Observe that
\begin{align*}
b_L(t_{k+1}) &= b_L(t_k) + \eta \cdot (l_1 - l), \\
b_U(t_{k+1}) &= b_U(t_k) + \eta \cdot (u_1 - u).
\end{align*}

\textbf{Lower boundary.} We apply an induction argument.
Assume that $\tilde\psi^k \ge b_L(t_k)$. Then
\begin{align*}
\tilde\psi^{k+1}-b_L(t_{k+1})
&=
\tilde\psi^k - b_L(t_k)
+
\eta\Big(v_{\theta^*}(\tilde\psi^k,t_k)
-
(l_1-l)
+
\xi_k
\Big).
\end{align*}

By Assumption~\ref{assump:boundary_condition}, there exists $\delta_g>0$ such that
\begin{align*}
v_{\theta^*}(\bz,t) \ge (l_1-l) + \delta_g
\quad \text{for } (\bz, t) \in \mathcal{L}_{K,\varepsilon_g}.
\end{align*}

Thus whenever $\tilde\psi^k$ lies in the lower boundary strip,
\begin{align*}
\tilde\psi^{k+1}-b_L(t_{k+1})
\ge
\tilde\psi^k - b_L(t_k)
+
\eta(\delta_g - \|\xi_k\|).
\end{align*}

If $\|\xi_k\| \le \bar\xi < \delta_g$, the increment is positive and the trajectory cannot cross the lower boundary.

\textbf{Upper boundary.}
We apply a similar inductive argument here:
\begin{align*}
\tilde\psi^{k+1}-b_U(t_{k+1})
&=
\tilde\psi^k - b_U(t_k)
+
\eta\Big(
v_{\theta^*}(\tilde\psi^k,t_k)
-
(u_1-u)
+
\xi_k
\Big).
\end{align*}

By Assumption~\ref{assump:boundary_condition}, there exists $\delta_g>0$ such that
\begin{align*}
v_{\theta^*}(\bz,t) \le (u_1-u) - \delta_g
\quad \text{on }~~ (\bz, t) \in \mathcal U_{K,\varepsilon_g}.
\end{align*}

Thus if $\|\xi_k\| \le \bar\xi < \delta_g$, the trajectory cannot cross the upper boundary. Since $\tilde\psi^0 \in [l,u]$, induction implies
\begin{align*}
(\tilde\psi^k,t_k) \in \mathcal C_K
\quad \forall k.
\end{align*}
This means that as long as the perturbations $\bar{\xi}$ are not too large, the trajectories stay within the conic region $\mathcal{C}_K$, thereby proving our result.
\end{proof}

\vspace{-0.2cm}
\subsection{Main Result}
\vspace{-0.2cm}

We now prove that the $c$-conic condition induces polynomial decay of perturbations.

\begin{tcolorbox}[
  colback=figurebackground,
  boxrule=0pt,
  left=6pt,
  right=6pt,
  top=6pt,
  bottom=6pt
]
\begin{theorem}[Polynomial Test-Time Recovery]
\label{thm:conic_implies_ttr_strong}
Assume that the learned velocity field satisfies Assumption~\ref{assump:smooth} (velocity smoothness), Assumption~\ref{assump:boundary_condition} (boundary condition), and the $c$-conic condition (Definition~\ref{def:conic}) for some $0 < c < 1$.
    % \begin{align*}
    % \frac{\partial v_{\theta^*}(\bz,t)}{\partial \bz}
    % \le
    % -\frac{c}{1-t}
    % \quad \forall (\bz,t)\in\mathcal C_K,
    % \end{align*}
    % for some $0<c<1$.
Let $\{\psi^k\}$ and $\{\tilde\psi^k\}$ be the unperturbed and perturbed trajectories with step size $\eta=1/K$. If $\max_k \|\xi_k\| \le \bar\xi$, so that both trajectories remain in $\mathcal C_K$ (Lemma~\ref{thm:trajectory_containment}), then
\begin{align*}
\|\tilde\psi^K - \psi^K\|
\le
C K^{-c}
\max_k \|\xi_k\|,
\end{align*}
for a constant $C>0$ independent of $K$. This implies that $\beta_K = O(K^{-c})$.
% In particular,
% \begin{align*}
% \beta_K = O(K^{-c}).
% \end{align*}
\end{theorem}
\end{tcolorbox}

\begin{proof}
Define $\delta_k := \tilde\psi^k - \psi^k$. Subtracting the updates,
\begin{align*}
\delta_{k+1}
&=
\delta_k
+
\eta
\Big(
v_{\theta^*}(\tilde\psi^k,t_k)
-
v_{\theta^*}(\psi^k,t_k)
\Big)
+
\eta \xi_k.
\end{align*}
By the mean value theorem, there exists $\zeta_k$ between $\tilde\psi^k$ and $\psi^k$ such that
\begin{align*}
v_{\theta^*}(\tilde\psi^k,t_k)
-
v_{\theta^*}(\psi^k,t_k)
=
\partial_\bz v_{\theta^*}(\zeta_k,t_k)
\, \delta_k.
\end{align*}

Since both trajectories lie in $\mathcal C_K$, the $c$-conic condition gives
\begin{align*}
\frac{\partial}{\partial \bz} v_{\theta^*}(\bz,t_k) \big|_{\bz = \zeta_k}
\le
-\frac{c}{1-t_k}.
\end{align*}
Thus
\begin{align*}
\delta_{k+1}
\le
\left(
1-\frac{\eta c}{1-t_k}
\right)\delta_k
+
\eta \xi_k.
\end{align*}
Since $t_k = k/K$,
\begin{align*}
\frac{\eta}{1-t_k}
=
\frac{1/K}{(K-k)/K}
=
\frac{1}{K-k}.
\end{align*}
Hence
\begin{align*}
\delta_{k+1}
\le
\left(
1-\frac{c}{K-k}
\right)\delta_k
+
\frac{1}{K}\xi_k.
\end{align*}
Unrolling the recursion yields
\begin{align*}
\delta_K
=
\sum_{j=0}^{K-1}
\left(
\prod_{i=j+1}^{K-1}
\left(1-\frac{c}{K-i}\right)
\right)
\frac{1}{K}\xi_j.
\end{align*}
Changing variables $m=K-i$, the product becomes
\begin{align*}
\prod_{m=1}^{K-j-1}
\left(1-\frac{c}{m}\right).
\end{align*}
Using the classical guarantee that
\begin{align*}
\prod_{m=1}^{M}
\left(1-\frac{c}{m}\right)
=
\Theta(M^{-c}),
\end{align*}
we obtain
\begin{align*}
\left|
\prod_{i=j+1}^{K-1}
\left(1-\frac{c}{K-i}\right)
\right|
\le
C (K-j)^{-c}.
\end{align*}
\begin{align*}
\text{Therefore,}~~~~
|\delta_K|
\le
\frac{C}{K}
\sum_{j=0}^{K-1}
(K-j)^{-c}
\max_k |\xi_k|.
\end{align*}
\begin{align*}
\text{Since}~~~\sum_{m=1}^{K} m^{-c}
=
\Theta(K^{1-c}),
\end{align*}
\begin{align*}
\text{we conclude}
~~~|\delta_K|
\le
C K^{-c}
\max_k |\xi_k|.
\end{align*}
This proves the result.
\end{proof}

\vspace{-0.3cm}
\subsection{Why Flow Matching Encourages the Conic Condition}
\label{sec:why_conic_holds}
\vspace{-0.1cm}

We now argue that the $c$-conic condition should be naturally satisfied in practice.

\textbf{Training objective.} For computing TD targets, the implementation of flow-matching critics we use (\texttt{floq} from \citet{agrawalla2025floqtrainingcriticsflowmatching}) generates $n$ initial noise samples $z \sim \mathrm{Unif}[l,u]$, integrates the target velocity field, and averages the resulting terminal target values. The critic is then trained using a flow-matching objective toward a Dirac distribution centered at the TD target. This induces a supervision signal that enforces transport from a broad initial distribution to a concentrated (Dirac-delta) terminal distribution.

\textbf{Implicit contraction.} Flow matching learns a velocity field that transports mass from the initial noise interval $[l,u]$ to a narrower target region $[l_1(\bs,\ba),u_1(\bs,\ba)]$. Since the target distribution has strictly smaller support, the learned flow must contract distances between nearby particles along the trajectory as long as it can be fit. As \citet{agrawalla2025floqtrainingcriticsflowmatching} notes, the values of TD error attained by the flow-matching critic are comparable to that of a standard critic implying no hardness of fitting the critic.
In the space of one-dimensional scalar values, such transport necessarily requires satisfying 
\begin{align*}
\frac{\partial v_{\theta^*}(\bz,t)}{\partial \bz} < 0
\end{align*}
over most of the supervision region, otherwise integration trajectories could expand and amplify errors.

\textbf{Time-varying contraction rate.} The flow interpolant linearly shrinks the admissible region for $\bz$ as $t$ increases. To ensure particles remain within this shrinking region, the required contraction strength must scale inversely with the remaining interval width, which is proportional to $(1-t)$. This yields a contraction rate of order $(1-t)^{-1}$, matching the $(c,\delta)$-conic condition.

\textbf{Dense supervision along the trajectory.} Unlike standard regression objectives that only constrain terminal outputs, flow matching supervises the velocity field at all intermediate times and across a broad distribution of interpolant inputs $\bz$. This provides local corrective signals through the conic region and suppresses expansion directions, making violations of the contraction condition rare.

\textbf{Effect of averaging over multiple noise samples.} Averaging terminal values over multiple initial noise samples further stabilizes the learned transport map by reducing variance in the terminal TD target. This improves the probability that the learned velocity field remains in the conic region. In fact, \textcolor{lightblue}{\textbf{\emph{the conic condition perhaps also implies that performing distributional RL with flow matching is perhaps less likely to result in convergent dynamics}}} compared to fitting an expected target, and may not enjoy TTR. 

\textbf{Conclusion.} Together, contraction of TD targets, dense trajectory supervision, and noise averaging encourage the learned velocity field to contract trajectories at a rate proportional to $(1-t)^{-1}$ over most of the conic region, implying that the $c$-conic condition holds with high probability.

\vspace{-0.2cm}
\subsection{Possible Extension of our Analysis by Relaxing the Conic and Boundary Conditions}
\label{sec:high_prob_relaxation}
\vspace{-0.2cm}

While the $c$-conic condition is expected to be satisfied in practice, throughout our analysis, the $c$-conic condition (Definition~\ref{def:conic}) and the boundary condition (Assumption~\ref{assump:boundary_condition}) were assumed to hold \emph{uniformly} over the entire supervision region $\mathcal C_K$. This uniform formulation simplifies the exposition and enables a clean deterministic contraction argument. That said, we can relax this condition to \emph{a high probability} condition over the supervision region. In particular, we can consider a relaxed setting in which
\begin{itemize}[itemsep=5pt]
    \item the $c$-conic inequality
    \begin{align*}
    \frac{\partial}{\partial \bz} v_{\theta^*}(\bz,t)
    \le
    -\frac{c}{1-t}
    \end{align*}
    holds for $(\bz,t)$ drawn from the supervision distribution with probability at least $1-\delta$, and
    \item the inward-pointing boundary condition holds with probability at least $1-\delta'$ over the boundary.
\end{itemize}
Under such a relaxation, our deterministic trajectory containment (Lemma~\ref{thm:trajectory_containment}) and contraction arguments (Theorem~\ref{thm:conic_implies_ttr_strong}) no longer apply verbatim, since adversarially arranged violations of the inequalities could, in principle, destroy contraction. For example, one can construct pathological velocity fields that satisfy the inequalities on a $(1-\delta)$ fraction of the region but systematically push almost all initial $\bz$ values toward the boundary through the exceptional low probability set, thereby preventing test-time recovery.

Nevertheless, such worst-case constructions are highly structured. In a probabilistic view of the hypothesis space, restricted to velocity fields that satisfy the $(1-\delta)$-probability versions of the conic and boundary conditions, vector fields that consistently steer typical trajectories through the exceptional set form a very small subset. Intuitively, doing so would require coordinated alignment between:
\begin{enumerate}
    \item the geometry of the exceptional set where contraction fails, and
    \item the discrete Euler trajectory induced by the learned flow.
\end{enumerate}
A more refined analysis could therefore proceed by: \textbf{(1)} showing that, with high probability over initialization $\bz \sim \mathrm{Unif}[l,u]$, the induced trajectory spends only a small fraction of its steps in regions where the conic inequality fails, and \textbf{(2)} controlling the cumulative effect of these rare expansion steps via a martingale or concentration argument.
One would then obtain a \emph{high-probability test-time recovery guarantee}, in which
\begin{align*}
\|\tilde\psi^K - \psi^K\|
\le
C K^{-c'} \max_k \|\xi_k\|
\end{align*}
holds with probability at least $1-\delta''$ over the initial noise and the trajectory, for suitably related $\delta''$.

We omit the formal argument and proof of this probabilistic contraction analysis, as it introduces substantial technical overhead without materially changing the conceptual picture. The key underlying point is that uniform contraction is sufficient but not necessary: it is enough that contraction holds along \emph{most} trajectories with high probability. As long as we can avoid adversarial fields, we should still enjoy concentration as we see in our experiments.

\section{Analysis of Feature Plasticity in a Linear Model (Section~\ref{subsec:toy_setting})}
\label{app:toy_setting}

In this appendix we provide proofs for Section~\ref{subsec:toy_setting} that discusses the mechanism behind feature plasticity in flow-matching critics using a linear setting. Our analysis revolves around evaluating the gradient flow dynamics of monolithic and flow-matching predictors that we formally describe below over the course of one gradient step of fitting a training TD target. This gradient flow model in a linear setting is akin to several prior works~\citep{arora2018optimization,du2018gradient}. While this model is simple, and makes assumptions (linearity, gradient flow), we find it to be sufficient to build a mental model of differences in behavior of the two types of architectures. 
In particular, we compare: \textbf{(1)} a \textbf{monolithic linear predictor} trained by a gradient flow with \textbf{(2)} a \textbf{linear Euler flow-matching model} trained with slice-wise supervision. Both represent linear functions, but their learning dynamics differ fundamentally due to dense supervision and the use of integration to compute the output allows the model to track new targets without needing to change features.
We summarize the parameterizations, resulting functions, training loss, and gradient flows for both type of critics in Table~\ref{tab:mono_vs_flow_linear_appendixE}.

\begin{table}[t]
\centering
\small
\setlength{\tabcolsep}{6pt}
\renewcommand{\arraystretch}{1.18}

\begin{tabularx}{\linewidth}{@{}>{\raggedright\arraybackslash}p{0.36\linewidth} >{\raggedright\arraybackslash}p{0.64\linewidth}@{}}
\toprule
\textbf{Monolithic critic} & 
\textbf{Flow-matching critic (Euler integration; slice-wise supervision)} \\
\midrule

% --------------------------------------------------
% FUNCTION
\multicolumn{2}{@{}l}{\textbf{Function}} \\[2pt]

\begin{minipage}[t]{\linewidth}\vspace{2pt}
Single linear predictor:
\[
f_{\mathrm{mono}}(\bx;m)=w(m)^\top \bx.
\]
\end{minipage}
&
\begin{minipage}[t]{\linewidth}\vspace{2pt}
$T$-step Euler recursion (integration index $t$, training step $m$):
\[
\!\!\!\!\!\!s_{i+1}(\bx,\bz;m)
=
(1+h\,v_i(m))\,s_i(\bx,\bz;m)
+
h\,u_i(m)^\top \bx,~~~
s_1=\bz.
\]

Output $\hat y(\bx,\bz;m)\coloneqq s_T(\bx,\bz;m)$.  
If $\E[\bz]=0$, the mean predictor is
\[
\!\!\!\!\!f_{\mathrm{FM}}(\bx;m)
=
\sum_{\ell=1}^{T-1}\beta_\ell(m)\,u_\ell(m)^\top \bx; ~~
\beta_\ell(m)
=
h\prod_{j=\ell+1}^{T-1}\bigl(1+h\,v_j(m)\bigr).
\]
\end{minipage}
\\[6pt]
\midrule

% --------------------------------------------------
% PARAMETERS
\multicolumn{2}{@{}l}{\textbf{Parameters}} \\[2pt]

\[
w(m)\in\mathbb{R}^d.
\]
&
For each slice $i\in\{1,\dots,T-1\}$:
\[
u_i(m)\in\mathbb{R}^d,
\quad
v_i(m)\in\mathbb{R},
\]
often grouped as $w_i(m)
\coloneqq
\begin{bmatrix}
u_i(m) \\
v_i(m)
\end{bmatrix}
\in\mathbb{R}^{d+1}.$
\\[6pt]

\midrule

% --------------------------------------------------
% LOSS
\multicolumn{2}{@{}l}{\textbf{Training loss}} \\[2pt]

Squared loss against time-varying target $Y(m)$:
\[
\mathcal{L}_{\mathrm{mono}}(m,w)
=
\E\big[(w^\top \bx -  y(m))^2\big].
\]
&
Slice-wise quadratic loss:
\[
\mathcal{L}_i(m, w_i)
=
\E\Big[
\big(
w_i^\top \tilde{\bx}_i(m)-(y(m)- \bz)
\big)^2
\Big],
\]
\[
\tilde{\bx}_i(m)
\coloneqq
\begin{bmatrix}
\bx \\
S_i(m)
\end{bmatrix},
~~
S_i(m)
=
\alpha_i \bz+(1-\alpha_i)y(m),
~~
\alpha_i
=
\frac{T-i}{T-1}.
\]
\\ \\

\midrule

% --------------------------------------------------
% GRADIENT FLOW
\multicolumn{2}{@{}l}{\textbf{Gradient flow}} \\[2pt]

Parameter dynamics, with
\[
\Sigma\coloneqq \E[\bx \bx^\top],
~~
b(m)\coloneqq \E[\bx\,y(m)]:
\]
\[
\dot w(m)
=
-2\big(\Sigma w(m)-b(m)\big),
\]
\[
\frac{\partial}{\partial m} f_{\mathrm{mono}}(\bx;m)
=
\dot w(m)^\top \bx.
\]
&
Per-slice parameter dynamics:
\[
\dot w_i(m)
=
-2\big(A_i(m)w_i(m)-b_i(m)\big).
\]

At the predictor level, define
\[
w_{\mathrm{eff}}(m)
\coloneqq
\sum_{\ell=1}^{T-1}\beta_\ell(m)\,u_\ell(m).
\]

Then the exact decomposition is
\[
\dot w_{\mathrm{eff}}(m)
=
\underbrace{
\sum_{\ell=1}^{T-1}\beta_\ell(m)\,\dot u_\ell(m)
}_{\text{feature learning}}
+
\underbrace{
\sum_{\ell=1}^{T-2}\dot\beta_\ell(m)\,u_\ell(m)
}_{\text{feature reweighting}},
\]
with
$\dot\beta_\ell(m)
=
\beta_\ell(m)
\sum_{k=\ell+1}^{T-1}
\frac{h\,\dot v_k(m)}
     {1+h\,v_k(m)}.$
\\
\bottomrule
\end{tabularx}

\caption{\footnotesize{
\textbf{Summary of differences between monolithic and flow-matching predictors in the linear gradient-flow model.}
Flow matching induces an additional \emph{feature reweighting channel}
(via $\dot\beta_\ell$, driven by gain dynamics $\dot v_k$)
that can move the predictor even when $\dot u_\ell=0$,
whereas monolithic predictors require direct motion of $w(m)$
to change $f_{\mathrm{mono}}(\cdot;m)$.
}}
\label{tab:mono_vs_flow_linear_appendixE}
\end{table}

\vspace{-0.2cm}
\subsection{Linear Euler Flow Matching Model}
\vspace{-0.2cm}

Fix $T \ge 3$ and step size $h \coloneqq \frac{1}{T-1}$. Consider the linear Euler recursion (integration index $t$):
\begin{align}
s_{i+1}(\bx,\bz;m)
&=
\bigl(1 + h\, v_i(m)\bigr)\, s_i(\bx,\bz;m)
+ h\, u_i(m)^\top \bx,
\qquad i = 1,\dots,T-1,
\\
s_1(\bx,\bz;m)
&=
\bz,
\qquad
\hat y(\bx,\bz;m)
\coloneqq
s_T(\bx,\bz;m).
\end{align}

The parameters at learning step $m$ are
\begin{align*}
u_i(m) \in \mathbb{R}^d,
\qquad
v_i(m) \in \mathbb{R},
\qquad i=1,\dots,T-1.
\end{align*}
\textbf{Interpretation.}
Each $u_i(t)$ injects a linear feature direction.
Each $v_i(t)$ is a \emph{gain parameter} that rescales accumulated signal. Next we write down the closed form of the linear predictor

\subsection{Unrolled Predictor Representation}

\begin{tcolorbox}[
  colback=figurebackground,
  boxrule=0pt,
  left=6pt,
  right=6pt,
  top=6pt,
  bottom=6pt
]
\begin{lemma}[Closed-form predictor]
Define
\begin{align}
P_{>i}(m)
&:= \prod_{j=i+1}^{T-1}\bigl(1+h\,v_j(m)\bigr), 
\qquad
\beta_i(m) := h\,P_{>i}(m).\\
\text{Then}~~~~~\hat y(\bx,\bz;m)
&=
\Big(\prod_{j=1}^{T-1}\bigl(1+h\,v_j(m)\bigr)\Big)\bz
+
\sum_{i=1}^{T-1} \beta_i(m)\,u_i(m)^\top \bx.
\end{align}
If $\E[\bz]=0$, the mean predictor is given by a linear function of $\bx$:
\begin{align}
f_{\mathrm{FM}}(\bx;m)
=
\sum_{i=1}^{T-1} \beta_i(m)\,u_i(m)^\top \bx.
\end{align}
\end{lemma}
\end{tcolorbox}

\begin{proof}
Unroll the Euler recursion:
\begin{align*}
s_{i+1}
=
\bigl(1+h\,v_i(m)\bigr)s_i
+
h\,u_i(m)^\top \bx.
\end{align*}

Iterating forward yields
\begin{align*}
s_T
=
\Big(\prod_{j=1}^{T-1}\bigl(1+h\,v_j(m)\bigr)\Big)\bz
+
\sum_{i=1}^{T-1}
\left(
h \prod_{j=i+1}^{T-1}\bigl(1+h\,v_j(m)\bigr)
\right)
u_i(m)^\top \bx,
\end{align*}
since each injection $u_i(m)^\top \bx$ is multiplied by all downstream gains.

By definition,
\(
\beta_i(m) = h\,P_{>i}(m),
\)
which gives the stated formula.

If $\E[\bz]=0$, the $\bz$-term vanishes in expectation, yielding
\begin{align*}
f_{\mathrm{FM}}(\bx;m)
=
\sum_{i=1}^{T-1}\beta_i(m)\,u_i(m)^\top \bx.
\end{align*}
\end{proof}
\emph{\textbf{Remark.}} If $\E[\bz] \neq 0$, then the term corresponding to $\bz$ appears as a \emph{constant} bias across all state-action pairs and independent of the feature weights $u_i(m)$.

\vspace{-0.2cm}
\subsection{Slice-wise Gradient Flow}
\vspace{-0.2cm}

Let $(\bx, y(m))$ be a time-indexed data process with finite second moments.
Assume $\E[\bz]=0$ and $\Var(\bz)=\sigma_z^2$. Each slice of integration is trained via a local squared error loss:
\begin{align}
\mathcal L_i(m,w_i)
&=
\E\Big[(w_i(m)^\top \tilde \bx_i(m) - (y(m)-\bz))^2\Big],
\qquad
w_i(m)
:=
\begin{bmatrix}
u_i(m) \\
v_i(m)
\end{bmatrix},\\
\text{where}~~~\tilde \bx_i(m)
:=&
\begin{bmatrix}
\bx \\
S_i(m)
\end{bmatrix},
\qquad
S_i(m)
:=
\alpha_i \bz + (1-\alpha_i) y(m), 
\qquad
\alpha_i
:=
\frac{T-i}{T-1}.
\end{align}
Also define the slice-wise second-moment matrix and cross-moment vector:
\begin{align}
A_i(m)
:=
\E\big[\tilde \bx_i(m)\tilde \bx_i(m)^\top\big],~~~~~b_i(m) :=
\E\big[\tilde \bx_i(m)\,(y(m)-\bz)\big].
\end{align}
Under gradient flow dynamics, the parameter dynamics satisfy the following relation at learning step $m$:
\begin{align*}
\dot w_i(m)
=
-2~\big(A_i(m) \cdot w_i(m) - b_i(m)~\big).
\end{align*}

\vspace{-0.2cm}
\subsection{Decomposition of Predictor Gradient Flow Dynamics}
\vspace{-0.2cm}
Define the effective weight vector as: $w_{\mathrm{eff}}(t) =
\sum_{i=1}^{T-1}\beta_i(t) u_i(t).$ Then, the following holds:

\begin{tcolorbox}[
  colback=figurebackground,
  boxrule=0pt,
  left=6pt,
  right=6pt,
  top=6pt,
  bottom=6pt
]
\begin{lemma}[Gradient flow for effective weight vector]
\label{thm:predictor_decomposition}
The time derivative of the predictor satisfies
\begin{align*}
\dot w_{\mathrm{eff}}(t)
=
\underbrace{\sum_{i=1}^{T-1}\beta_i(t)\,\dot u_i(t)}_{\textbf{Feature learning}}
+
\underbrace{\sum_{i=1}^{T-2}\dot\beta_i(t)\,u_i(t)}_{\textbf{Feature reweighting}},
~~~~~~~~~\text{where}~~~~\dot\beta_i(t)
=
\beta_i(t)
\sum_{k=i+1}^{T-1}
\frac{h\,\dot v_k(t)}{1+h v_k(t)}.
\end{align*}
\end{lemma}
\end{tcolorbox}
\begin{proof}
Differentiating the expression for $w_\mathrm{eff}(t)$ using the product rule, we get:
\begin{align*}
\dot w_{\mathrm{eff}}
&=
\sum_{i=1}^{T-1}
\big(
\dot\beta_i u_i
+
\beta_i \dot u_i
\big).
\end{align*}
Rearranging yields the stated decomposition. It remains to compute $\dot\beta_i$.
Recall
\begin{align*}
\beta_i
=
h \prod_{k=i+1}^{T-1}(1+h v_k).
\end{align*}
Taking logarithmic derivatives,
\begin{align*}
\frac{\dot\beta_i}{\beta_i}
&=
\sum_{k=i+1}^{T-1}
\frac{h \dot v_k}{1+h v_k}.
\end{align*}
Multiplying both sides by $\beta_i$ gives
\begin{align*}
\dot\beta_i
=
\beta_i
\sum_{k=i+1}^{T-1}
\frac{h \dot v_k}{1+h v_k}.
\end{align*}
\end{proof}

\vspace{-0.3cm}
\subsection{Gradient Flow for the Gain Parameters $v_k(m)$ Under Target Non-Stationarity}
\vspace{-0.2cm}

Define moments:
\begin{align*}
\Sigma := \E[\bx\bx^\top],
\qquad
b(m):=\E[\bx\, y(m)],
\qquad
q(m):=\E[y(m)^2].
\end{align*}

\begin{tcolorbox}[
  colback=figurebackground,
  boxrule=0pt,
  left=6pt,
  right=6pt,
  top=6pt,
  bottom=6pt
]
\begin{lemma}[Gradient flow for gain $v_k(m)$]
\label{thm:gain_dynamics}
For each slice $k$ of integration,
\begin{align*}
\dot v_k(m)
=
-2\Big(
(1-\alpha_k)\,b(m)^\top u_k(m)
+
\big((1-\alpha_k)^2 q(m)+\alpha_k^2\sigma_z^2\big)v_k(m)
-
(1-\alpha_k)q(m)
+
\alpha_k\sigma_z^2
\Big).
\end{align*}
\end{lemma}
\end{tcolorbox}
\begin{proof}
The loss for a given slice is given by:
\begin{align*}
\mathcal L_k(m)
=
\E\big[(u_k(m)^\top \bx + v_k(m)\, S_k(m) - (y(m)-\bz))^2\big].
\end{align*}

Differentiating w.r.t.\ $v_k(m)$,
\begin{align*}
\partial_{v_k}\mathcal L_k(m)
&=
2\E\big[
(u_k(m)^\top \bx + v_k(m)\, S_k(m) - (y(m)-\bz))\, S_k(m)
\big].
\end{align*}

Using
\(
S_k(m) = \alpha_k \bz + (1-\alpha_k)y(m),
\)
and $\E[\bz]=0$, $\E[\bz\,y(m)]=0$, $\Var(\bz)=\sigma_z^2$,
direct expansion yields
\begin{align*}
\E[S_k(m)^2]
&=
(1-\alpha_k)^2 q(m) + \alpha_k^2\sigma_z^2,
\\
\E[(y(m)-\bz)\,S_k(m)]
&=
(1-\alpha_k) q(m) - \alpha_k \sigma_z^2,
\\
\E[(u_k(m)^\top \bx)\, S_k(m)]
&=
(1-\alpha_k)\, b(m)^\top u_k(m).
\end{align*}

Substituting and applying gradient flow
\(
\dot v_k(m) = -2\,\partial_{v_k}\mathcal L_k(m)
\)
gives the stated expression.
\end{proof}

\vspace{-0.1cm}
\subsection{Feature Plasticity Under Moving Targets}
\vspace{-0.1cm}
Next, we put these intermediate results into a theorem that shows how flow-matching critics preserve feature plasticity compared to monolithic critics.

\begin{tcolorbox}[
  colback=figurebackground,
  boxrule=0pt,
  left=6pt,
  right=6pt,
  top=6pt,
  bottom=6pt
]
\begin{theorem}[Feature reweighting in flow-matching networks.]
\label{thm:feature_reweighting_clean}
Fix an interval $[m_0,m_1]$ and suppose $\dot u_\ell(m)=0 \quad \forall \ell,\ m\in[m_0,m_1].$ Then the following conditions hold:
\begin{enumerate}[itemsep=5pt]
\item \textbf{\textit{(Flow matching.)}} For the flow-matching predictor, $\dot w_{\mathrm{eff}}(m) =
\sum_{\ell=1}^{T-2}\dot\beta_\ell(m)\,u_\ell$,  so the predictor can evolve entirely via the dynamics of the gain parameter $\{\dot v_k(m)\}$ (Lemma~\ref{thm:gain_dynamics}).

\item \textbf{\textit{(Monolithic predictor.)}} For $f_{\mathrm{mono}}(\bx;m)=w(m)^\top \bx$ trained by a squared loss,
\begin{align*}
\dot w(m)
=
-2(\Sigma w(m)-b(m)).
\end{align*}
Thus changing any predictions to chase a new target requires $\dot w(m)\neq 0$. When $w(m)$ is parameterized by a deep linear or ResNet architecture, this requires alternating \emph{at least} one $u_i(m)$.
\end{enumerate}
Hence flow matching can admit predictor adaptation entirely via \emph{reweighting existing features}, whereas monolithic networks requires modifying at least some feature vector.
\end{theorem}
\end{tcolorbox}
\begin{proof}
If $\dot u_\ell(m)=0$ on $[m_0,m_1]$, then
Lemma~\ref{thm:predictor_decomposition} reduces to
\begin{align*}
\dot w_{\mathrm{eff}}(m)
=
\sum_{\ell=1}^{T-2}\dot\beta_\ell(m)\,u_\ell,
\end{align*}
which depends only on gain dynamics. On the other hand, for the monolithic predictor,
\begin{align*}
\mathcal L(w)
& =
\E[(w^\top \bx - y(m))^2].\\
\text{Differentiating this}~~~~~\dot w(m)
& =
-2(\Sigma w(m) - b(m)).
\end{align*}
Thus predictor motion requires $\dot w(m)\neq 0$, i.e.\ at least one feature vector itself must move since $w(m) = \prod_{t=1}^T u_t(m)$ for a deep linear network or $w(m) = \prod_{t=1}^T (I + u_t(m))$ for a deep residual network.

Further note that $\dot w_\text{eff}(m) \neq 0$ since $\frac{\dot v_k(m)}{v_k(m)} = f(u_k(m), v_k(m))$. When $u_k(m)$ and $v_k(m)$ take on non-zero values, $\dot v_k(m) \neq 0$. Therefore, the function prediction can change just via the process of integration.
\end{proof}

\vspace{-0.2cm}
\subsection{Why Fixed-Weight Ensembles Cannot Replicate Reweighting}
\label{sec:monolithic_linear_new}
\vspace{-0.1cm}

We also study whether feature ensembles of monolithic critics can attain a similar effect as flow matching. To study this, consider $K$ monolithic predictors
\begin{align*}
f^{(p)}(\bx;m)=w_p(m)^\top \bx,
\end{align*}
trained independently by gradient flow but from different initializations of $w_p(m)$. Next, assume that we combine them with fixed weights $\pi_p$ that satisfy $\sum_k \pi_k = 1$. We use this to define an ensemble predictor:
\begin{align*}
f_{\mathrm{ens}}(\bx;m)
=
\sum_{p=1}^P \pi_p f^{(p)}(\bx;m).
\end{align*}
Equipped with this definition, we now show that at least some feature vector within a network needs to change to make changes to the ensemble predictor. 

\begin{tcolorbox}[
  colback=figurebackground,
  boxrule=0pt,
  left=6pt,
  right=6pt,
  top=6pt,
  bottom=6pt
]
\begin{lemma}[Ensemble collapse]
\label{thm:ensemble_collapse}
Let $\bar w(m)
:=
\sum_{k=1}^K \pi_k w_k(m).$
Then
\begin{align*}
f_{\mathrm{ens}}(\bx;m)
&= \bar w(m)^\top \bx,
~~~~~~~\text{and}~~~~~~
\dot{\bar w}(m) =
-2(\Sigma \bar w(m)-b(m)).
\end{align*}

In particular, if $\dot w_k(m)=0$ for all $k$ on an interval,
then $f_{\mathrm{ens}}$ is constant.
\end{lemma}
\end{tcolorbox}
\begin{proof}
By linearity,
\begin{align*}
f_{\mathrm{ens}}(\bx;m)
=
\sum_k \pi_k w_k(m)^\top \bx
=
\bar w(m)^\top \bx.
\end{align*}
Differentiating,
\begin{align*}
\dot{\bar w}(m)
=
\sum_k \pi_k \dot w_k(m).
\end{align*}

Each $w_k$ follows gradient flow:
\begin{align*}
\dot w_k(m)
=
-2(\Sigma w_k(m) - b(m)).
\end{align*}
Summing,
\begin{align*}
\dot{\bar w}(m)
&=
-2\sum_k \pi_k (\Sigma w_k(m) - b(m)) \\
&=
-2(\Sigma \bar w(m) - b(m)).
\end{align*}
If all $\dot w_k(m)=0$ on an interval,
then $\dot{\bar w}(m)=0$,
so the ensemble predictor is constant.
\end{proof}

\vspace{-0.1cm}
\subsection{Interpretation and Takeaways}
\vspace{-0.1cm}

Although both models represent linear predictors,
flow matching decomposes learning into two channels:
\begin{itemize}
\item \textbf{Feature learning:}
updates of $u_m(t)$.

\item \textbf{Feature reweighting:}
updates of downstream gains $v_k(t)$
which modify coefficients $\beta_m(t)$.
\end{itemize}
From a high-level perspective, under non-stationary targets, the gain dynamics (Lemma~\ref{thm:gain_dynamics}) responds directly to
moment drift in $b(t)$ and $q(t)$, allowing rapid reallocation of
importance among previously learned features to track the non-stationary target. This explains our intuition that changes in the TD target are absorbed by the integration process without changing the features themselves in the case of a flow-matching critic. Monolithic predictors and fixed-weight ensembles lack this sort of an adaptation procedure. They must overwrite feature parameters $u_k(m)$ to track target drift. Note that this advantage is a direct consequence of the dense supervision (Lemma~\ref{thm:predictor_decomposition}) and the use of integration at test time. This formalizes the plasticity advantage of flow matching.

\end{document}